\documentclass[11pt]{article}

\usepackage[preprint]{acl}
\usepackage{booktabs}
\usepackage{times}
\usepackage{booktabs,multirow}
\usepackage{latexsym}
\usepackage{booktabs}
\usepackage{booktabs}
\usepackage{graphicx}
\usepackage{pdflscape}
\usepackage[T1]{fontenc}
\usepackage[utf8]{inputenc}
\usepackage{microtype}
\usepackage{xspace}
\IfFileExists{inconsolata.sty}{\usepackage{inconsolata}}{}
\usepackage{graphicx}
\usepackage{amsmath,amssymb,amsthm}
\newtheorem{theorem}{Theorem}
\newtheorem{proposition}[theorem]{Proposition}
\newtheorem{corollary}[theorem]{Corollary}
\theoremstyle{remark}

\usepackage[table]{xcolor}
\definecolor{lightgray}{gray}{0.9}
\definecolor{lightpink}{RGB}{255,238,243}

\newcommand{\method}{\texttt{CaC}\xspace}
\newcommand{\ub}[1]{\textcolor{gray}{#1}}  
\usepackage{multirow}
\usepackage{graphicx}
\usepackage{makecell}
\newcolumntype{C}{>{\centering\arraybackslash}p{2.6em}}

\usepackage{makecell}
\newcolumntype{C}{>{\centering\arraybackslash}p{2.6em}}
\NewDocumentCommand{\shujin}
{ mO{} }{\textcolor{pink}{\textsuperscript{\textit{shujin}}\textsf{\textbf{\small[#1]}}}}
\title{Free Lunch for LLMs Confidence: The RL Critic Already Knows}

\author{
  Shujin Wu, Heng Ji \\
  University of Illinois Urbana-Champaign \\
  \texttt{\{shujinwu, hengji\}@illinois.edu}  \\
}

\begin{document}
\maketitle

\begin{abstract}
% \looseness=-1
Reliable confidence estimation is essential for trustworthy deployment of large language models (LLMs), but dedicated confidence training stage requires additional data and compute. We ask a fundamental question: what prediction problem does confidence training actually solve? Inspired by prior work on external confidence estimation, we compare self- and external estimators of comparable capacity trained on the same (question, response, correctness) triples. We show that an external model trained solely on a policy's behavioral traces can estimate the policy's confidence as accurately as the policy itself, without accessing the policy's internal representations. This points to a deeper connection with value function learning in reinforcement learning (RL): under a correctness-based outcome, we prove that confidence estimation and critic training are equivalent, both predicting the likelihood that a policy's response will be correct from its trajectory. A simple consequence follows: for LLMs trained with value-based RL, the existing critic can be reused as an off-the-shelf confidence estimator with no additional training or data. The policy remains unchanged, so its task performance is preserved, and the critic provides confidence throughout generation, including intermediate reasoning steps. Across math, code, tool calling, and general reasoning spanning a dozen domains, the reused critic is the best confidence estimator on every in-distribution split and the best-calibrated on 14 of 16, without hurting the task performance.

\end{abstract}
\section{Introduction}

Large language models (LLMs) can produce incorrect answers even when their responses appear fluent and convincing. In many applications, it is therefore important to know not only what an LLM predicts, but also how likely its prediction is to be correct \citep{jiang-etal-2021-know,mielke-etal-2022-reducing}. This motivates the problem of confidence calibration: a confidence estimate should match the probability that the model's response is correct \citep{guo2017calibration,jiang-etal-2021-know}. Reliable confidence estimation is, consequently, an important component of trustworthy LLMs deployment.

% Large language models (LLMs) are increasingly deployed in settings where knowing \emph{when to trust a model's output} is as important as producing accurate outputs. A model that is highly accurate on average can still be unreliable if its confidence does not reflect the likelihood that a particular response is correct. This distinction between accuracy and confidence calibration has long been recognized in machine learning: a calibrated predictor should assign probabilities that correspond to the empirical likelihood of being correct \citep{guo2017calibration}. For LLMs, reliable confidence estimates are particularly important because models can produce fluent and convincing responses even when they are incorrect, making output uncertainty difficult to assess \citep{jiang2021calibration,farquhar2024semantic}.
\looseness=-1
A growing line of work explicitly trains LLMs to estimate or express their confidence. 
For example, R-Tuning~\citep{zhang-etal-2024-r} constructs training data to teach LLMs to recognize the boundaries of their parametric knowledge and respond with ``I don't know.'' SaySelf~\citep{xu-etal-2024-sayself} further trains LLMs to produce fine-grained confidence estimates, using automatically constructed supervision for supervised fine-tuning followed by reinforcement learning (RL). 
While these methods demonstrate that confidence estimation can be improved through dedicated training, 
they introduce an extra training stage, requiring extra data and compute beyond the typical LLMs post-training pipeline. 
This motivates us to ask a more fundamental question: \emph{\textbf{what learning problem does confidence training actually solve?}}

% This motivates us to ask whether well-calibrated confidence can instead be obtained directly from a model's existing behavior, without requiring a separate confidence-training procedure.

We approach this question by first examining what information is sufficient for learning confidence estimation.
Inspired by prior work on external confidence estimation \citep{ulmer2024calibrating,kapoor2024taught}, we systematically compare self- and external confidence estimators with comparable model capacity, trained on the same behavioral traces (\textit{i.e.,} question--response--correctness triples).
Our empirical experiments demonstrate that without access to the internal representations of the target policy, \emph{\textbf{an external model can learn to estimate the confidence of an LLM policy as accurately as the policy itself, using only the observable behavior of the policy.}}
%

% We empirically demonstrate that external estimators can achieve AUROC comparable to self-estimation using only question–response–correctness triples,  or token probabilities. These results support learning confidence from observable behavior alone and motivate a deeper connection to value function learning.

% In this work, we start with a surprising finding: \emph{\textbf{an external model can learn to estimate an LLM policy's confidence as accurately as the policy itself, using only the policy's observable behavior.}}
% Specifically, given behavioral traces in the form of , an external model trained solely on these traces achieves confidence estimation performance comparable to that of the target LLM policy's own confidence estimates. 
% Notably, the external model has no access to the target policy's internal representations or token-level probabilities. 
% The results demonstrates that well-calibrated confidence can be learned and estimated directly from a policy's observable behavior.

This finding leads us to a deeper connection between confidence estimation and value function learning in reinforcement learning (RL).
In value-based reinforcement learning, a critic learns to predict the expected outcome of a policy's behavior based on the information available along its trajectory~\citep{suttonbarto2018}. 
For an LLM solving a math problem, for example, this information consists of the question together with the solution generated so far. When the outcome is binary correctness, the critic's prediction corresponds to the likelihood that the final response will be correct, which is precisely what a confidence estimator is trained to capture. 
For continuous correctness or quality rewards, the same connection holds with confidence interpreted as the expected outcome under the chosen reward. 
Thus, confidence estimation and critic training can be understood as learning the same prediction problem: predicting the likelihood that a policy's outcome will be correct, given the information available along its trajectory. 
We formalize and prove this equivalence in Section~\ref{sec:theory}.

This perspective leads to a striking practical consequence: \emph{\textbf{for LLMs trained with value-based RL, the confidence estimator is already embedded in the RL training pipeline.}} In particular, for algorithms such as Proximal Policy Optimization (PPO)~\citep{schulman2017ppo}, the critic trained during RL can be directly reused as an off-the-shelf confidence estimator. This reuse has several practical advantages. First, it requires no additional confidence-specific training or data collection, avoiding the extra optimization stage that may otherwise alter the original policy. Second, the critic is trained throughout the RL process on the policy's trajectories, providing supervision across a broad range of states encountered during generation rather than only at the final response. This enables the same critic to estimate confidence not only for complete responses, but also for intermediate reasoning steps. Finally, because the critic learns from the policy's behavior rather than from task-specific confidence annotations, it can be naturally transferred across downstream tasks. 
Thus, a value function trained for RL can be repurposed as a general-purpose confidence estimator without modifying the original policy.

We evaluate this approach, Critic-as-Confidence (\textbf{\method}), on two LLMs across math, code, agentic tool calling, and general reasoning, spanning a dozen domains. Against both training-free and training-based baselines, \method is the best discriminator on every in-distribution split and the best-calibrated on 14 of 16, while fine-tuning the policy for confidence costs 3.9 accuracy points on average. 
The critic is also useful beyond confidence estimation. Selecting the highest-confidence sample among eight improves accuracy by 4.1 to 5.6 points, and reading the critic at intermediate prefixes gives confidence before and during generation, enabling early termination.
We also find that the recipe extends to critic-free RL, where a value model fit post hoc on GRPO rollouts is again the best estimator. 

\looseness=-1
Our contribution goes beyond turning a critic into a confidence estimator. 
We establish a conceptual bridge between confidence estimation and value function learning, revealing that what is often treated as a separate confidence estimation problem can, under a correctness-based outcome, be viewed as the same prediction problem already solved by RL critics. 
This perspective offers a unified way to understand and develop confidence estimation from policy trajectories, and invites a broader rethinking of how existing RL training signals and value functions can be repurposed for confidence estimation.
\section{Confidence Estimation from Observable Behavior: Do We Need Internal Representations?}
\label{sec:internal-access}

% \looseness=-1
To understand whether a separate confidence training stage is necessary for teaching LLMs to give reliable confidence estimation, we first examine the prediction problem underlying confidence estimation and the supervision it requires. 
We begin by testing whether confidence can be learned from a policy's observable behavior alone.

% We begin with a preliminary experiment to investigate an important question: \emph{Do we need internal representations of an LLM policy to estimate its confidence?}
\looseness=-1
Following the general training-based paradigm of prior confidence estimation methods~\citep{zhang-etal-2024-r}, we train an external model with comparable capacity using only question--response--correctness triples collected from the target policy. 
This training relies exclusively on the policy's observable behavior.
We then test whether an external model can estimate the target policy's confidence as accurately as the policy itself.
Notably, the external model has no access to the target policy's internal representations or token-level probabilities.

\paragraph{Experimental setup.}
We select Llama-3-8B-Instruct~\citep{grattafiori2024llama3} as the target policy $\pi_\theta$.
For each benchmark, we sample responses $y \sim \pi_\theta(\cdot \mid x)$ from the training split and automatically evaluate their correctness $c \in \{0,1\}$ using the corresponding benchmark-specific grader.
This yields a behavioral dataset
\[
\mathcal{D} = \{(x_i, y_i, c_i)\}_{i=1}^{N}.
\]
We cap the number of training examples at 2,000 per benchmark.
We evaluate on five benchmarks spanning three task families: GSM8K~\citep{cobbe2021training} and NuminaMath~\citep{numinamath2024} for mathematical reasoning, APPS-intro~\citep{hendrycks2021apps} and TACO~\citep{li2023taco} for code generation, and ARC-Challenge~\citep{clark2018arc} for multiple-choice science reasoning.

\paragraph{Confidence estimation with R-Tuning.}
We adopt R-Tuning~\citep{zhang-etal-2024-r} as a representative training-based confidence estimator.
Similar with related methods~\citep{lin2022teaching,kapoor2024taught,xu-etal-2024-sayself}, it supervises a model with correctness labels derived from the target policy's own responses, so our findings are not specific to its particular output format.
Following R-Tuning, we convert each behavioral example into a supervised fine-tuning target by appending a confidence statement to the response.
For a correct response, we append ``I am sure'' to the response, while for an incorrect response, we append ``I am unsure''.
At inference time, we use the probability assigned to the confidence token as a continuous confidence score.
We train three models with similar capacities on the same behavioral dataset: the target Llama-3-8B-Instruct model for self-estimation, and Qwen2.5-7B-Instruct~\citep{qwen2025qwen25} and Mistral-7B-Instruct-v0.3~\citep{jiang2023mistral} as external estimators.
All models are fine-tuned with the same LoRA-based~\citep{hu2022lora} supervised fine-tuning procedure.
The key difference is that the self-estimation model has access to its own internal representations~\citep{kadavath2022know,azaria2023internal}, whereas the external estimators receive only the question and the target policy's response as text.

\paragraph{Main results.}
\looseness=-1
Table~\ref{tab:self_vs_external} reports AUC (AUROC) for correctness discrimination, together with expected calibration error (ECE) \citep{naeini2015ece} and Brier score \citep{brier1950} for evaluating the quality of the predicted probabilities.
External estimators achieve comparable discrimination overall: Qwen2.5-7B-Instruct stays within 0.05 AUROC of the self-estimator on four of the five benchmarks, and an external estimator achieves the highest AUROC, including ties, on four benchmarks.
The probability-based metrics provide further evidence that strong confidence estimation can be learned from observable behavior: on every benchmark, at least one external estimator matches or improves upon the self-estimator's ECE, and the same holds for Brier score on three of the five benchmarks.

% Table~\ref{tab:self_vs_external} reports AUROC, our primary metric for evaluating whether the confidence score discriminates correct from incorrect responses.
% External estimators perform on par with the self-estimator throughout: Qwen2.5-7B-Instruct stays within 0.05 AUROC of it on four of the five benchmarks, and on four benchmarks an external estimator even achieves the highest AUROC.

\begin{table}[t]
\centering
\footnotesize
\renewcommand{\arraystretch}{1.12}
\setlength{\tabcolsep}{2.5pt}
\begin{tabular*}{\columnwidth}{
  @{\extracolsep{\fill}}lccccc@{}
}
\toprule
\textbf{Estimator}
& \textbf{GSM8K}
& \textbf{Numina}
& \textbf{ARC-C}
& \textbf{APPS}
& \textbf{TACO} \\
\midrule

\multicolumn{6}{@{}l}{AUC $\uparrow$} \\
\addlinespace[2pt]
Self
& 0.71 & \textbf{0.83} & \textbf{0.67} & 0.84 & 0.89 \\
Qwen-7B
& \textbf{0.75} & \textbf{0.83} & 0.63 & 0.82 & 0.80 \\
Mistral-7B
& 0.71 & 0.68 & 0.58 & \textbf{0.86} & \textbf{0.90} \\

\midrule
\multicolumn{6}{@{}l}{ECE $\downarrow$} \\
\addlinespace[2pt]
Self
& 0.21 & 0.08 & 0.07 & \textbf{0.05} & \textbf{0.06} \\
Qwen-7B
& 0.10 & \textbf{0.07} & \textbf{0.04} & \textbf{0.05} & 0.11 \\
Mistral-7B
& \textbf{0.09} & 0.31 & 0.26 & 0.11 & \textbf{0.06} \\

\midrule
\multicolumn{6}{@{}l}{Brier $\downarrow$} \\
\addlinespace[2pt]
Self
& 0.20 & \textbf{0.13} & 0.18 & \textbf{0.11} & \textbf{0.02} \\
Qwen-7B
& \textbf{0.16} & 0.14 & \textbf{0.15} & 0.12 & 0.03 \\
Mistral-7B
& 0.17 & 0.27 & 0.22 & \textbf{0.11} & 0.04 \\

\bottomrule
\end{tabular*}
\caption{Confidence estimation on the responses from
Llama-3-8B-Instruct using self and external estimators.}
\vspace{-13pt}
\label{tab:self_vs_external}
\end{table}

% Thus, despite having no privileged access to the target policy's internal states, an external model trained solely on behavioral traces recovers the same correctness-discrimination signal as the target model trained on the same data.

\paragraph{Conclusion.}
\looseness=-1
These results show that privileged access to a target policy's internal representations is not necessary for learning a strong confidence signal: the information required for confidence estimation is already contained in the policy's observable input--output behavior.
This reframes confidence estimation as an external prediction problem:
\[
f(x,y) \rightarrow P(c=1 \mid x,y).
\]
The next section shows that, under a correctness-based reward, this is precisely the prediction problem learned by the value function in reinforcement learning~\citep{suttonbarto2018}.

\section{Confidence Estimation as Value Function Learning}
\label{sec:theory}
% \loo
% \looseness=-1
Section~\ref{sec:internal-access} raises a natural question: why can a policy's confidence be estimated from its observable behavior alone? The key intuition is simple: confidence is the probability that the policy's response will eventually be correct.
This probability is determined by how the policy behaves and how its outcomes are evaluated. 
The policy's internal representations may provide useful features for estimating this probability, but they do not change what we are trying to estimate.
Under a correctness-based reward, this probability is exactly the value function that a critic learns in value-based RL. 
We formalize this connection below and defer the full proofs to Appendix~\ref{app:theory}.

\paragraph{Setup.}
Let $x$ be a prompt and $y = (y_1, \dots, y_T)$ a response generated token by token from $\pi_\theta$. The state at step $t$ is the prefix $s_t = (x, y_{1:t})$; $s_T = (x, y)$ is the complete pair. A grader assigns a correctness label $c \in \{0, 1\}$ to $(x, y)$, and RL uses it as the terminal reward, $R(x, y) = c$, with no intermediate rewards. We write $\mathbb{E}_\pi$ for expectation over $(x, y, c)$ with $y \sim \pi_\theta(\cdot \mid x)$.
 
\paragraph{Confidence estimation.}
A calibrator $f_\phi$ is trained on behavioral traces to predict correctness,
\begin{equation}
  \min_\phi\; \mathbb{E}_\pi\big[\mathcal{L}\big(f_\phi(x, y),\, c\big)\big],
  \label{eq:calib}
\end{equation}
where $\mathcal{L}$ is the log loss or the squared loss. Both are strictly proper scoring rules, so the optimal calibrator $f^*$ is the conditional probability of correctness, $f^*(x, y) = \mathbb{P}_\pi(c = 1 \mid x, y)$. The same objective applies to a prefix $s_t$ in place of $(x, y)$.
 
\paragraph{Value function.}
A critic $V_\psi$ is trained to predict the expected return from a state. With a terminal correctness reward, the return from every prefix is $R(x, y)$ itself, so the Monte Carlo objective is
\begin{equation}
  \min_\psi\; \mathbb{E}_\pi\Big[\big(V_\psi(s_t) - R(x, y)\big)^2\Big],
  \label{eq:critic}
\end{equation}
whose optimum is the true value function $V^\pi(s_t) = \mathbb{E}_\pi[R \mid s_t]$. 
In practice, PPO fits its critic to bootstrapped GAE targets rather than directly to Eq.~\eqref{eq:critic}, but these targets have the same fixed point (Appendix~\ref{app:td}).

\paragraph{Equivalence.}
Because $R=c$ is binary, the expected return is simply the probability of a correct final response. Thus, the two problems make the same prediction at every step of the trajectory.
 
\begin{theorem}[informal; see Theorem~\ref{thm:formal} in Appendix~\ref{app:theory}]
\label{thm:main}
Under a terminal correctness reward, for every step $t$,
\begin{equation}
  V^\pi(s_t) \;=\; f^*(s_t) \;=\; \mathbb{P}_\pi\big(c = 1 \mid s_t\big).
  \label{eq:equiv}
\end{equation}
With the squared loss, Eqs.~\eqref{eq:calib} and~\eqref{eq:critic} are the same objective on the same data and so have the same minimizers for any model class; with the log loss, they share the same optimum whenever the model class can represent it. The shared solution is perfectly calibrated, attains the highest calibration performance among all functions of the prefix, and is consistent along the trajectory, $\mathbb{E}_\pi[V^\pi(s_{t+1}) \mid s_t] = V^\pi(s_t)$.
\end{theorem}
 
The only difference between a calibrator and a critic is therefore \emph{when} it is trained.
The calibrator is trained post hoc, as a dedicated module on freshly collected traces,
and the critic is trained during RL, as a byproduct of policy optimization, on every state of every rollout. For LLMs post-trained with a value-based method such as PPO~\citep{schulman2017ppo}, VC-PPO~\citep{yuan2025vcppo} or VAPO~\citep{yue2025vapo}, the critic is already available.

Three unique advantages are derived from the theory. 
First, by Theorem~\ref{thm:main} the critic's target is the correctness probability at \emph{every} prefix, and these estimates are consistent along the trajectory (Appendix~\ref{app:props}).
The critic thus provides confidence before generation starts ($s_0$, the prompt alone), at intermediate steps, and at the final response, whereas a calibrator trained on $(x, y, c)$ triples is supervised only at $s_T$. Second, the critic's supervision is drawn from the policy's own rollouts throughout RL, so it needs no separate data collection and covers every trajectory of training rather than a capped sample. Third, since the estimator is external to the policy (Theorem~\ref{thm:main} never refers to its parameters), the policy is left untouched and its task performance cannot be degraded by confidence training, unlike methods that fine-tune the policy itself to express confidence \citep{zhang-etal-2024-r, xu-etal-2024-sayself}.

\paragraph{Scope of the equivalence.}
The connection holds for any fixed policy, allows stochastic graders, and is independent of the estimator's architecture. 
With the same model class and data, the squared-loss objectives are identical.
More generally, any strictly proper loss recovers the same correctness probability whenever the model class can represent it. Bootstrapped GAE targets also preserve $V^\pi$ as their unique population fixed point for every $\lambda \in [0,1]$ (Appendix~\ref{app:td}).

\looseness=-1
Applying this equivalence to a trained critic requires considering three aspects of the RL setup (Appendix~\ref{app:scope}).
First, the critic must estimate the undiscounted terminal correctness return ($R=c$, $\gamma=1$). KL regularization may be applied as a separate policy loss, whereas including KL penalties or other shaping terms in the reward changes the value target.
Second, the critic estimates correctness under a particular policy and its decoding settings, so it should be used with the corresponding policy and training-time decoding settings. 
Third, the learning objectives are defined under the RL prompt distribution,
and whether the trained critic transfers to other tasks is an empirical question addressed in Section~\ref{sec:experiments}.

\paragraph{Critic-as-Confidence (\method).}
Given a policy $\pi_\theta$ and the critic $V_\psi$ from the same RL run, \method uses the critic, unchanged, as the confidence estimator: $\hat{p}(x, y) = \mathrm{clip}(V_\psi(x, y), 0, 1)$ for a complete response and $\hat{p}(x, y_{1:t}) = \mathrm{clip}(V_\psi(x, y_{1:t}), 0, 1)$ for a partial one, read off at the last generated token. 
The clipping here only guards against regression outputs slightly outside $[0, 1]$.
\method needs no confidence-specific data, no additional training, and no modification of the policy, whose task performance is preserved by construction.
The extra inference-time cost is only one critic forward pass.
 
% ============================================================
% Section 4: Experiments (main text; ACL two-column)
% Preamble needs: \usepackage{booktabs,multirow}
% Labels assumed elsewhere in the paper: sec:external (Section 2),
%   sec:theory (Section 3), thm:main, eq:\method, app:scope (Appendix A.5).
% Labels defined here: sec:experiments, sec:exp-setup, sec:exp-main,
%   sec:exp-ood, tab:main, tab:general.
% Appendix labels used: app:exp-details, app:full-results, tab:hparams,
%   tab:full-math, tab:full-code, tab:full-tool, tab:full-general.
% [TODO: ...] marks facts to confirm before submission.
% All numbers are generated from the results sheet (AUROC/ECE/Brier x100).
% ============================================================

\section{Experiments}
\label{sec:experiments}

\begin{table*}[t]
\centering
\footnotesize
\renewcommand{\arraystretch}{1.1}
\setlength{\tabcolsep}{5pt}
\begin{tabular}{c l *{9}{C}}
\toprule
& & \multicolumn{2}{c}{Math} & \multicolumn{2}{c}{Code} & \multicolumn{2}{c}{\makecell[b]{Tool Calling}} & \multicolumn{2}{c}{General} & \\
\cmidrule(lr){3-4} \cmidrule(lr){5-6} \cmidrule(lr){7-8} \cmidrule(lr){9-10}
& Method & ID & OOD & ID & OOD & ID & OOD & ID & OOD & Avg. \\
\midrule
& R-Tuning & 38.5 & 70.5 & 49.7 & 57.5 & 47.2 & 59.5 & 26.1 & 54.0 & 50.4 \\
& Original Policy & 43.6 & 73.4 & 51.6 & 58.2 & 52.2 & 59.1 & 33.8 & 62.6 & 54.3 \\
\cmidrule(l){2-11}
& Maj@8 & 43.9 & 78.3 & 53.7 & 60.1 & \textbf{52.3} & 60.5 & 35.5 & 67.1 & 56.4 \\
\rowcolor{lightpink} \cellcolor{white} & \method{} Selection (ours) & \textbf{44.1} & \textbf{80.3} & \textbf{55.2} & \textbf{63.4} & \textbf{52.3} & \textbf{62.1} & \textbf{38.7} & \textbf{71.2} & \textbf{58.4} \\
\multirow{-5}{*}{\rotatebox[origin=c]{90}{Qwen3-8B}}
& \ub{Best@8 (oracle)} & \ub{45.7} & \ub{86.7} & \ub{60.7} & \ub{69.4} & \ub{52.5} & \ub{64.4} & \ub{44.8} & \ub{74.5} & \ub{62.3} \\
\midrule
& R-Tuning & 49.1 & 65.9 & 53.2 & 59.9 & 48.6 & 57.3 & 23.1 & 55.6 & 51.6 \\
& Original Policy & 54.5 & 72.9 & 55.9 & 60.8 & 49.4 & 59.0 & 30.1 & 61.2 & 55.5 \\
\cmidrule(l){2-11}
& Maj@8 & 55.2 & 78.3 & 56.6 & 63.2 & \textbf{51.2} & 65.2 & 37.2 & 64.8 & 59.0 \\
\rowcolor{lightpink} \cellcolor{white} & \method{} Selection (ours) & \textbf{55.7} & \textbf{81.2} & \textbf{60.1} & \textbf{64.2} & 51.1 & \textbf{68.4} & \textbf{38.5} & \textbf{69.7} & \textbf{61.1} \\
\multirow{-5}{*}{\rotatebox[origin=c]{90}{\makecell{Nemotron-\\Cascade-8B}}}
& \ub{Best@8 (oracle)} & \ub{56.5} & \ub{85.0} & \ub{64.2} & \ub{70.5} & \ub{58.8} & \ub{70.9} & \ub{42.8} & \ub{76.1} & \ub{65.6} \\
\bottomrule
\end{tabular}
\caption{Task accuracy (\%) with and without inference-time scaling.
Rows above the thin rule use a single response; rows below sample $N{=}8$ responses, where \method{} Selection returns the one with the highest critic-estimated confidence and Best@8 is the oracle upper bound.}
\vspace{-6pt}
\label{tab:acc}
\end{table*}

\begin{table*}[t!]
\centering
\footnotesize
\renewcommand{\arraystretch}{1.25}
\setlength{\tabcolsep}{1.8pt}
\begin{tabular}{l ccc ccc ccc ccc ccc ccc}
\toprule
& \multicolumn{6}{c}{Math} & \multicolumn{6}{c}{Code} & \multicolumn{6}{c}{Tool Calling} \\
\cmidrule(lr){2-7} \cmidrule(lr){8-13} \cmidrule(lr){14-19}
& \multicolumn{3}{c}{ID} & \multicolumn{3}{c}{OOD} & \multicolumn{3}{c}{ID} & \multicolumn{3}{c}{OOD} & \multicolumn{3}{c}{ID} & \multicolumn{3}{c}{OOD} \\
\cmidrule(lr){2-4} \cmidrule(lr){5-7} \cmidrule(lr){8-10} \cmidrule(lr){11-13} \cmidrule(lr){14-16} \cmidrule(lr){17-19}
Method & AUC & ECE & Brier & AUC & ECE & Brier & AUC & ECE & Brier & AUC & ECE & Brier & AUC & ECE & Brier & AUC & ECE & Brier \\
\midrule
\multicolumn{19}{l}{\textbf{Qwen3-8B}} \\
\quad Verbalized & 0.66 & 0.56 & 0.53 & 0.73 & 0.19 & 0.21 & 0.60 & 0.37 & 0.38 & 0.61 & 0.28 & 0.26 & 0.47 & 0.37 & 0.39 & 0.57 & 0.30 & 0.33 \\
\quad Self-consistency & 0.70 & 0.44 & 0.42 & \textbf{0.94} & 0.16 & \textbf{0.06} & 0.78 & 0.22 & 0.23 & 0.89 & 0.10 & 0.10 & 0.71 & 0.38 & 0.36 & 0.91 & 0.30 & 0.30 \\
\quad R-Tuning & 0.89 & 0.05 & \textbf{0.13} & 0.88 & 0.24 & 0.18 & 0.86 & 0.05 & 0.15 & 0.85 & 0.08 & 0.11 & 0.95 & 0.07 & 0.08 & 0.90 & 0.19 & 0.16 \\
\quad Probe & 0.88 & 0.05 & 0.15 & 0.89 & 0.24 & 0.17 & 0.88 & 0.06 & 0.14 & 0.84 & 0.10 & 0.12 & 0.93 & 0.12 & 0.10 & 0.85 & 0.15 & 0.18 \\
\rowcolor{lightpink} \quad \method (ours) & \textbf{0.91} & \textbf{0.03} & 0.14 & \textbf{0.94} & \textbf{0.13} & 0.15 & \textbf{0.93} & \textbf{0.04} & \textbf{0.11} & \textbf{0.91} & \textbf{0.06} & \textbf{0.08} & \textbf{0.97} & \textbf{0.06} & \textbf{0.05} & \textbf{0.92} & \textbf{0.11} & \textbf{0.15} \\
\midrule
\multicolumn{19}{l}{\textbf{Nemotron-Cascade-8B}} \\
\quad Verbalized & 0.62 & 0.47 & 0.45 & 0.70 & 0.26 & 0.28 & 0.58 & 0.33 & 0.35 & 0.60 & 0.32 & 0.28 & 0.56 & 0.32 & 0.36 & 0.46 & 0.28 & 0.35 \\
\quad Self-consistency & 0.83 & 0.23 & 0.22 & 0.97 & 0.18 & 0.19 & 0.76 & 0.24 & 0.25 & \textbf{0.88} & 0.20 & 0.16 & 0.81 & 0.22 & 0.22 & \textbf{0.77} & 0.24 & 0.28 \\
\quad R-Tuning & 0.85 & 0.04 & 0.15 & 0.88 & 0.20 & 0.17 & 0.89 & \textbf{0.03} & 0.13 & 0.79 & 0.09 & \textbf{0.09} & 0.86 & 0.04 & 0.15 & 0.72 & 0.23 & 0.26 \\
\quad Probe & 0.84 & 0.04 & 0.15 & 0.87 & 0.19 & 0.18 & 0.89 & \textbf{0.03} & 0.13 & 0.80 & 0.09 & \textbf{0.09} & 0.86 & 0.04 & 0.14 & 0.74 & 0.30 & 0.30 \\
\rowcolor{lightpink} \quad \method (ours) & \textbf{0.87} & \textbf{0.03} & \textbf{0.14} & \textbf{0.99} & \textbf{0.17} & \textbf{0.14} & \textbf{0.91} & \textbf{0.03} & \textbf{0.12} & 0.84 & \textbf{0.08} & \textbf{0.09} & \textbf{0.88} & \textbf{0.03} & \textbf{0.13} & \textbf{0.77} & \textbf{0.21} & \textbf{0.24} \\
\bottomrule
\end{tabular}
\caption{Confidence estimation performance on math, code and agentic tool calling for two PPO-trained LLMs. 
\textbf{AUC is AUROC (higher is better), ECE and Brier are lower-is-better.} 
ID and OOD values are averages over the sub-benchmarks of each split weighted by the number of scored samples (per-benchmark results in Appendix~\ref{app:full-results}).}
\vspace{-8pt}
\label{tab:main}
\end{table*}

\begin{table*}[t]
\centering
\footnotesize
\renewcommand{\arraystretch}{1.2}
\setlength{\tabcolsep}{2.2pt}
\begin{tabular}{l l cccccccc c ccccc c}
\toprule
& & \multicolumn{9}{c}{In distribution} & \multicolumn{6}{c}{Out of distribution} \\
\cmidrule(lr){3-11} \cmidrule(lr){12-17}
& Method & Phys. & Math & Bus. & Chem. & Fin. & Econ. & Hist. & Bio. & All & Phys. & Bio. & Med. & Multi & Chem. & All \\
\midrule
\multicolumn{17}{l}{\textbf{Qwen3-8B}} \\
\multirow{5}{*}{AUC $\uparrow$} & Verbalized & 0.59 & 0.55 & 0.52 & 0.52 & 0.53 & 0.53 & 0.36 & 0.48 & 0.55 & 0.69 & 0.52 & 0.60 & 0.56 & 0.55 & 0.59 \\
 & Self-consistency & 0.79 & 0.80 & 0.64 & 0.76 & 0.75 & \textbf{0.82} & 0.77 & 0.73 & 0.76 & 0.84 & 0.71 & \textbf{0.82} & 0.81 & \textbf{0.76} & 0.79 \\
 & R-Tuning & 0.77 & 0.75 & 0.63 & 0.74 & 0.68 & 0.67 & 0.84 & 0.77 & 0.73 & 0.87 & 0.81 & 0.70 & 0.64 & 0.61 & 0.73 \\
 & Probe & 0.81 & 0.76 & 0.63 & 0.76 & 0.68 & 0.66 & \textbf{0.87} & 0.78 & 0.75 & 0.82 & 0.71 & 0.67 & 0.62 & 0.59 & 0.68 \\
\rowcolor{lightpink} \cellcolor{white} & \method (ours) & \textbf{0.98} & \textbf{0.88} & \textbf{0.69} & \textbf{0.81} & \textbf{0.78} & 0.75 & \textbf{0.87} & \textbf{0.82} & \textbf{0.86} & \textbf{0.93} & \textbf{0.83} & 0.74 & \textbf{0.83} & 0.66 & \textbf{0.80} \\
\cmidrule(lr){2-17}
\multirow{5}{*}{ECE $\downarrow$} & Verbalized & 0.77 & 0.56 & 0.64 & 0.71 & 0.56 & 0.37 & 0.65 & 0.72 & 0.66 & 0.77 & 0.21 & 0.21 & 0.28 & 0.36 & 0.37 \\
 & Self-consistency & 0.28 & 0.18 & 0.34 & 0.25 & 0.13 & 0.19 & 0.31 & 0.36 & 0.25 & 0.22 & \textbf{0.16} & \textbf{0.10} & 0.15 & 0.14 & 0.15 \\
 & R-Tuning & 0.03 & 0.05 & 0.14 & 0.05 & 0.08 & 0.18 & 0.09 & 0.19 & 0.07 & 0.05 & 0.27 & 0.19 & 0.18 & 0.12 & 0.16 \\
 & Probe & 0.03 & 0.06 & 0.15 & 0.07 & 0.08 & 0.19 & 0.10 & 0.18 & 0.08 & \textbf{0.03} & 0.30 & 0.18 & 0.15 & 0.15 & 0.16 \\
\rowcolor{lightpink} \cellcolor{white} & \method (ours) & \textbf{0.02} & \textbf{0.04} & \textbf{0.12} & \textbf{0.04} & \textbf{0.06} & \textbf{0.15} & \textbf{0.07} & \textbf{0.11} & \textbf{0.06} & \textbf{0.03} & 0.21 & 0.17 & \textbf{0.13} & \textbf{0.08} & \textbf{0.12} \\
\cmidrule(lr){2-17}
\multirow{5}{*}{Brier $\downarrow$} & Verbalized & 0.72 & 0.55 & 0.60 & 0.67 & 0.53 & 0.38 & 0.63 & 0.67 & 0.62 & 0.71 & 0.25 & 0.25 & 0.31 & 0.37 & 0.38 \\
 & Self-consistency & 0.21 & 0.21 & 0.34 & 0.23 & 0.19 & 0.22 & 0.28 & 0.30 & 0.23 & 0.16 & 0.21 & \textbf{0.15} & \textbf{0.19} & \textbf{0.22} & 0.19 \\
 & R-Tuning & 0.09 & 0.19 & 0.21 & 0.14 & 0.21 & 0.23 & 0.14 & 0.16 & 0.16 & 0.10 & 0.23 & 0.22 & 0.25 & 0.25 & 0.21 \\
 & Probe & 0.08 & 0.19 & 0.22 & 0.14 & 0.20 & 0.24 & 0.13 & 0.15 & 0.15 & 0.10 & 0.28 & 0.22 & 0.25 & 0.27 & 0.22 \\
\rowcolor{lightpink} \cellcolor{white} & \method (ours) & \textbf{0.03} & \textbf{0.16} & \textbf{0.20} & \textbf{0.11} & \textbf{0.18} & \textbf{0.20} & \textbf{0.11} & \textbf{0.14} & \textbf{0.12} & \textbf{0.07} & \textbf{0.18} & 0.20 & \textbf{0.19} & 0.23 & \textbf{0.17} \\
\midrule
\multicolumn{17}{l}{\textbf{Nemotron-Cascade-8B}} \\
\multirow{5}{*}{AUC $\uparrow$} & Verbalized & 0.54 & 0.61 & 0.55 & 0.58 & 0.54 & 0.57 & 0.49 & 0.49 & 0.56 & 0.58 & 0.73 & 0.57 & 0.59 & 0.56 & 0.61 \\
 & Self-consistency & 0.83 & 0.80 & 0.69 & 0.75 & 0.71 & 0.73 & 0.70 & \textbf{0.80} & 0.78 & \textbf{0.87} & 0.77 & \textbf{0.69} & 0.74 & 0.66 & 0.74 \\
 & R-Tuning & 0.80 & 0.76 & 0.67 & 0.77 & 0.69 & 0.77 & 0.75 & 0.77 & 0.76 & 0.80 & 0.77 & 0.64 & 0.64 & 0.61 & 0.69 \\
 & Probe & 0.83 & 0.80 & 0.69 & \textbf{0.81} & 0.72 & 0.78 & \textbf{0.81} & 0.76 & 0.79 & 0.73 & \textbf{0.78} & 0.62 & 0.63 & 0.59 & 0.67 \\
\rowcolor{lightpink} \cellcolor{white} & \method (ours) & \textbf{0.95} & \textbf{0.91} & \textbf{0.73} & 0.80 & \textbf{0.74} & \textbf{0.81} & \textbf{0.81} & 0.77 & \textbf{0.86} & 0.83 & \textbf{0.78} & 0.67 & \textbf{0.79} & \textbf{0.68} & \textbf{0.75} \\
\cmidrule(lr){2-17}
\multirow{5}{*}{ECE $\downarrow$} & Verbalized & 0.74 & 0.47 & 0.56 & 0.72 & 0.54 & 0.41 & 0.68 & 0.68 & 0.62 & 0.63 & 0.13 & 0.21 & 0.25 & 0.25 & 0.29 \\
 & Self-consistency & 0.21 & 0.17 & 0.22 & 0.31 & 0.07 & 0.16 & 0.37 & 0.28 & 0.21 & 0.15 & 0.16 & 0.12 & \textbf{0.17} & 0.14 & 0.15 \\
 & R-Tuning & \textbf{0.04} & 0.05 & 0.08 & 0.06 & 0.08 & 0.12 & \textbf{0.12} & \textbf{0.15} & \textbf{0.06} & 0.06 & \textbf{0.10} & 0.21 & 0.24 & 0.23 & 0.17 \\
 & Probe & 0.05 & \textbf{0.04} & 0.10 & \textbf{0.05} & 0.07 & \textbf{0.08} & 0.13 & 0.20 & \textbf{0.06} & 0.06 & 0.27 & 0.17 & 0.26 & 0.14 & 0.18 \\
\rowcolor{lightpink} \cellcolor{white} & \method (ours) & 0.05 & \textbf{0.04} & \textbf{0.07} & 0.06 & \textbf{0.05} & \textbf{0.08} & 0.15 & 0.22 & \textbf{0.06} & \textbf{0.05} & 0.11 & \textbf{0.11} & 0.19 & \textbf{0.12} & \textbf{0.12} \\
\cmidrule(lr){2-17}
\multirow{5}{*}{Brier $\downarrow$} & Verbalized & 0.67 & 0.44 & 0.52 & 0.66 & 0.51 & 0.41 & 0.62 & 0.62 & 0.57 & 0.56 & 0.19 & 0.25 & 0.29 & 0.31 & 0.32 \\
 & Self-consistency & 0.14 & 0.20 & 0.27 & 0.26 & \textbf{0.15} & 0.19 & 0.35 & 0.23 & 0.20 & 0.13 & 0.19 & \textbf{0.18} & 0.23 & \textbf{0.24} & 0.19 \\
 & R-Tuning & 0.08 & 0.18 & 0.20 & 0.12 & 0.19 & 0.19 & 0.14 & 0.18 & 0.14 & 0.13 & 0.17 & 0.25 & 0.28 & 0.29 & 0.22 \\
 & Probe & \textbf{0.07} & 0.17 & 0.20 & 0.12 & 0.19 & 0.19 & \textbf{0.13} & \textbf{0.16} & 0.14 & 0.13 & 0.24 & 0.25 & 0.31 & 0.27 & 0.24 \\
\rowcolor{lightpink} \cellcolor{white} & \method (ours) & \textbf{0.07} & \textbf{0.14} & \textbf{0.19} & \textbf{0.10} & 0.18 & \textbf{0.18} & 0.15 & \textbf{0.16} & \textbf{0.12} & \textbf{0.12} & \textbf{0.16} & 0.19 & \textbf{0.21} & \textbf{0.24} & \textbf{0.18} \\
\bottomrule
\end{tabular}
\caption{Confidence estimation performance on general reasoning benchmarks spanning multiple domains. LLMs are trained with PPO on WebInstruct-verified.}
% \vspace{-8pt}
% AUC is AUROC (higher is better); ECE and Brier are lower-is-better. ``All'' is the sample-weighted average over the domains of the split. Best per column within each model and metric in bold. History (80 samples) and Biology (56) are small slices.
\label{tab:general}
\end{table*}

% We evaluate \method on two LLMs and multiple domains, in and out of distribution.

\subsection{Setup}
\label{sec:exp-setup}

\paragraph{Datasets.}
We train Qwen3-8B \citep{qwen3} and Nemotron-Cascade-8B \citep{nemotroncascade} with PPO \citep{schulman2017ppo}, one run per domain: 
math on NuminaMath \citep{numinamath2024}, code on TACO, CodeContests and Codeforces \citep{li2023taco, li2022alphacode, codeforces_data},
and agentic tool calling on ToolRL \citep{qian2025toolrl}.
We also run mix-RL experiments on a general reasoning dataset WebInstruct-verified \citep{ma2025generalreasoner} that spans multiple domains. 
\method reads the critic at the last response token and clips it to $[0,1]$.

\looseness=-1
\paragraph{Baselines.}
\emph{Verbalized} \citep{tian2023just, xiong2024llms} asks the policy for the probability that its answer is correct. 
\emph{Self-consistency} \citep{wang2023selfconsistency} uses the agreement rate among 8 samples, at $8\times$ the inference cost. 
\emph{R-Tuning} \citep{zhang-etal-2024-r}, the training-based estimator, fine-tunes the policy on held-out labeled responses to append ``I am sure'' or ``I am unsure''. 
\emph{Probe} \citep{azaria2023internal, kadavath2022know} fits a linear classifier on the policy's last-token activations with the same labeled data. 
The labeled sets are disjoint from all RL training and evaluation data, and each contains the same number of examples as the corresponding RL training set.

\paragraph{Evaluation.}
Each domain has an in-distribution (ID) split of held-out problems from the RL training sources and an out-of-distribution (OOD) split from external benchmarks: 
AIME 2024/2025 \citep{aime} for math, LiveCodeBench \citep{jain2024livecodebench} for code, 
API-Bank \citep{li2023apibank} for tool calling.
For general reasoning, the OOD split draws 400 questions each from MedMCQA \citep{pal2022medmcqa}, PubMedQA \citep{jin2019pubmedqa}, ChemBench \citep{zhang2024chemllm}, UGPhysics \citep{xu2025ugphysics}, and OpenScienceReasoning-2 \citep{nvidia2025openscience2}.
Every prompt is answered with several samples at the RL temperature, and every sample is scored with AUROC (AUC), 15-bin ECE and the Brier score. Split-level numbers are averages over sub-benchmarks weighted by sample count,
and per-benchmark results are in Appendix~\ref{app:full-results}. 
We also describe the additional training and evaluation details in Appendix~\ref{app:exp-details}.

\subsection{Results}
\label{sec:exp-main}
Table~\ref{tab:acc} reports task accuracy, Table~\ref{tab:main} reports confidence estimation on math, code and agentic tool calling, 
and Table~\ref{tab:general} reports general reasoning by domain. 
The lower half of Table~\ref{tab:acc} (inference-time scaling) is discussed in Section~\ref{sec:scale}.

\paragraph{\method preserves the policy's task accuracy.}
\method never updates the policy, so its task accuracy is exactly that of the RL policy (the \emph{Original Policy} rows in Table~\ref{tab:acc}). R-Tuning does not share this property: fine-tuning the policy on the confidence objective lowers accuracy on 15 of the 16 model--domain--split combinations, by 3.9 points on average for both Qwen3-8B and Nemotron-Cascade-8B. 
This drop is the cost of a dedicated confidence-training stage. \method pays nothing here because it never updates the policy: the critic is trained alongside it during RL and is read out as is.

\looseness=-1
\paragraph{The critic is the best discriminator in distribution and competitive out of it.}
\method has the highest AUROC on all eight ID model--domain splits, ahead of R-Tuning by 0.02 to 0.13 and of Probe by 0.02 to 0.11, with the largest margins on general reasoning, the most heterogeneous domain (0.86 for both LLMs against 0.76 and 0.79 for the best baselines). 
On out of distribution benchmarks, \method is best on seven of eight splits.
The only ties, on AIME for Qwen3-8B and tool calling OOD for Nemotron-Cascade-8B, are with Self-consistency, which needs $8\times$ the inference cost.

\paragraph{The raw values computed by \method are calibrated probabilities.}
With no post-processing beyond clipping, \method has the lowest ECE on 14 of 16 splits and ties for lowest on the other two. 
On in distribution benchmarks, \method's ECE never exceeds 0.06, where R-Tuning reaches 0.07, Probe 0.12 and the training-free methods 0.21 to 0.66. 
\method also has the best or tied-best Brier score on 14 of 16 splits. Verbalized confidence is near chance for both LLMs (AUROC 0.46 to 0.73), in line with reports that RL inflates stated confidence \citep{leng2025taming}.

\paragraph{The advantage widens under distribution shift.}
\label{sec:exp-ood}
Averaged over the eight model--domain pairs, ECE rises from ID to OOD by 0.12 for both R-Tuning and Probe, but by 0.08 for \method.
On AIME with Qwen3-8B, R-Tuning goes from 0.05 to 0.24 while \method goes from 0.03 to 0.13. 
A plausible reason for the critic's robustness is the breadth of its supervision: every prefix of every rollout over the whole RL run, rather than a few thousand final responses.

\subsection{Inference-Time Scaling with \method}
\label{sec:scale}

Since the critic estimates $P(c{=}1\mid x,y)$, it can be used to select among samples for inference-time scaling: for each prompt, we draw $N{=}8$ responses at the RL temperature and return the one with the highest critic value (\method Selection). We compare against majority voting over the same samples (Maj@8) and the oracle Best@8. The lower half of Table~\ref{tab:acc} reports these inference-time scaling results.
\method Selection improves the single-sample policy by 4.1 points on Qwen3-8B and 5.6 on Nemotron-Cascade-8B on average, closing about half of the gap to the oracle, where Maj@8 closes 26\% and 35\%. 
Unlike majority voting, the critic scores each sample independently, so it can select a correct minority answer and needs no task-specific rule for grouping equivalent outputs, indicating \method Selection is an effective test-time scaling approach.

\begin{table}[t]
\centering
\footnotesize
\renewcommand{\arraystretch}{1.1}
\setlength{\tabcolsep}{3pt}
\resizebox{\columnwidth}{!}{%
\begin{tabular}{l *{6}{C}}
\toprule
& \multicolumn{3}{c}{ID} & \multicolumn{3}{c}{OOD} \\
\cmidrule(lr){2-4} \cmidrule(lr){5-7}
Method & AUC & ECE & Brier & AUC & ECE & Brier \\
\midrule
Verbalized       & 0.53 & 0.60 & 0.58 & 0.55 & 0.28 & 0.30 \\
Self-consistency & 0.67 & 0.41 & 0.41 & 0.78 & 0.18 & 0.21 \\
R-Tuning         & 0.87 & 0.05 & 0.14 & 0.79 & 0.25 & 0.24 \\
Probe            & 0.86 & 0.03 & 0.14 & 0.74 & 0.17 & 0.22 \\
\rowcolor{lightpink} \method{} (ours) & \textbf{0.90} & \textbf{0.02} & \textbf{0.12} & \textbf{0.81} & \textbf{0.14} & \textbf{0.20} \\
\bottomrule
\end{tabular}%
}
\caption{Confidence estimation on general reasoning for a Qwen3-8B policy trained with GRPO. \method{} uses a value model fit post hoc on the GRPO rollouts.}
\label{tab:grpo}
\vspace{-8pt}
\end{table}

\subsection{Extension to Critic-Free RL}
\label{sec:exp-grpo}

Critic-free methods such as GRPO~\citep{shao2024deepseekmath} produce no value function to read out, but the equivalence in Section~\ref{sec:theory} does not depend on how the policy is optimized. 
We train Qwen3-8B with GRPO on WebInstruct-verified and fit a value model post hoc on the rollouts that GRPO already generated, using the same critic objective as in PPO and leaving the policy untouched. 
Table~\ref{tab:grpo} shows that this post-hoc critic is the best estimator on every metric and split (AUROC 0.90 ID and 0.81 OOD, against 0.87 and 0.79 for the strongest baseline), mirroring the PPO results. \method{} thus extends to critic-free RL at the cost of one auxiliary training stage on existing data, and the comparison isolates the recipe from the algorithm: given the same policy, learning a value function on its rollouts yields better confidence estimation performance than confidence-specific fine-tuning or hidden-state probing.

\begin{table}[t]
\centering
\footnotesize
\renewcommand{\arraystretch}{1.1}
\setlength{\tabcolsep}{3pt}
\resizebox{\columnwidth}{!}{%
\begin{tabular}{l *{6}{C}}
\toprule
& \multicolumn{3}{c}{ID} & \multicolumn{3}{c}{OOD} \\
\cmidrule(lr){2-4} \cmidrule(lr){5-7}
Position & AUC & ECE & Brier & AUC & ECE & Brier \\
\midrule
0\% (prompt) & 0.83 & 0.09 & 0.14 & 0.73 & 0.18 & 0.22 \\
25\%  & 0.86 & 0.07 & 0.13 & 0.79 & 0.13 & 0.18 \\
50\%  & \textbf{0.87} & 0.07 & 0.13 & \textbf{0.80} & 0.13 & \textbf{0.17} \\
75\%  & 0.86 & \textbf{0.05} & \textbf{0.12} & \textbf{0.80} & 0.13 & \textbf{0.17} \\
\rowcolor{lightpink} 100\% & 0.86 & \textbf{0.05} & \textbf{0.12} & \textbf{0.80} & \textbf{0.12} & \textbf{0.17} \\
\bottomrule
\end{tabular}%
}
\caption{Further analysis of the final general reasoning Qwen3-8B checkpoint. Confidence along the trajectory. $V_\psi(s_t)$ is read at the prompt alone (0\%) and at 25\%, 50\%, 75\% and 100\% of each response's length.}
\vspace{-5pt}
\label{tab:traj}
\end{table}
\subsection{\method: Confidence along the Trajectory}
\label{sec:exp-qual}

Theorem~\ref{thm:main} demonstrates that the critic estimates $P(c{=}1\mid s_t)$ at every prefix, not only at the final response. 
We conduct a further analysis of the final general reasoning checkpoint after mix-RL. 
We read $V_\psi(s_t)$ at five positions along each response (Table~\ref{tab:traj}). 
Three observations follow. 
First, the critic is informative before a single token is generated: at the prompt alone it reaches AUROC 0.83 in distribution and 0.73 out of distribution, so it has learned which prompts the policy will get right.
Second, most of the remaining gain arrives in the first quarter of the response (AUROC 0.83 to 0.86 ID, 0.73 to 0.79 OOD), after which discrimination is flat to within 0.01.
The outcome is largely decided by the approach the policy commits to early. 
Third, calibration keeps improving along the trajectory (ECE 0.09 to 0.05 ID, 0.18 to 0.12 OOD) with no sign of drift, consistent with $\mathbb{E}[V^\pi(s_{t+1})\mid s_t]=V^\pi(s_t)$ in Corollary~\ref{cor:props}. A practical consequence is early termination: a response whose confidence is low at the halfway point can be abandoned at half the generation cost with little loss in ranking quality.
\section{Related Work}

% \paragraph{Confidence estimation for LLMs.}
Training-free confidence estimation for LLMs mainly includes
\emph{probability-based} methods using token likelihoods or entropy
\citep{kadavath2022know},
\emph{sampling-based} methods measuring consistency or semantic entropy
\citep{manakul2023selfcheckgpt, chen2023quantifying,
kuhn2023semantic, farquhar2024semantic},
and \emph{verbalized-confidence} methods eliciting confidence through
prompting \citep{tian2023just}.
Sampling increases inference cost, while elicited confidence can remain
poorly calibrated across tasks \citep{xiong2024llms}.
This has driven a recent shift toward
\emph{training-based} estimators, which learn to predict correctness directly
from supervision on the policy's own graded responses and now constitute the
strongest family of methods. \citet{kapoor2024taught} argue that prompting alone is insufficient and show that fine-tuning on as few as a thousand graded
examples outperforms training-free baselines and generalizes across domains.
R-Tuning \citep{zhang-etal-2024-r} teaches the model to append an ``I am
sure/unsure'' judgment.
SaySelf \citep{xu-etal-2024-sayself} fine-tunes the model to produce fine-grained confidence, followed by RL
with a calibration-aware reward.
APRICOT \citep{ulmer2024calibrating} trains an auxiliary
model to predict confidence from input/output text alone, and
\emph{hidden-state probes} fit lightweight classifiers on internal activations
\citep{azaria2023internal, kossen2024semantic,
zhang2025reasoning}. 
Because supervised estimators offer the best accuracy and
calibration, they are the focus of this work. Within this family, a key open
question is what the supervision signal actually requires. Probe assumes
white-box access to the target model, yet \citet{kapoor2024taught} and
\citet{ulmer2024calibrating} find that a separate model can estimate
\emph{another} model's uncertainty, and surrogate-model approaches show that
confidence transfers across models \citep{shrivastava2023llamas}. We make this
observation precise: given the same behavioral data, self- and external
estimation perform nearly identically, reducing the problem to learning
$f(x,y)\!\rightarrow\!P(\text{correct}\mid x,y)$.
This behavioral formulation suggests a connection to RL critics,
which learn to predict outcomes from policy trajectories.
Related empirical explorations use critic values as a baseline for behavioral calibration \citep{wu2025mitigating} and, concurrently,
for quantitative prediction with error-based rewards \citep{dhaliwal2026critic}.
In this work, we establish a unified theoretical foundation for the equivalence between confidence estimation and critic learning under undiscounted terminal correctness rewards, and demonstrate its practical value through large-scale RL experiments spanning a dozen domains in math, code, tool calling, and general reasoning.
Additional related work on value function learning for LLMs is discussed in Appendix~\ref{app:related}.

\section{Conclusion}
% \looseness=-1
We show that confidence estimation and value function learning are the same prediction problem under a correctness-based outcome, so the RL critic can be reused as a confidence estimator with no extra training. 
In an extensive evaluation that spans multiple domains, \method is the best estimator in a dozen domains and extends to intermediate steps, test-time selection, and critic-free RL.

\section*{Limitations}
\method requires a critic, and some of the RL recipes for LLMs (GRPO and other critic-free methods) do not train one. 
Our remedy, evaluated in Section~\ref{sec:exp-grpo}, is to fit a value model on the rollouts that GRPO already generates during training.
This reuses existing data and never touches the policy, but it does add a training stage, so the ``free lunch'' is exact only for value-based RL. 
Besides, the critic adds one forward pass of a value model per response.

\bibliography{custom}

@article{kadavath2022know,
  title = {Language Models (Mostly) Know What They Know},
  author = {Kadavath, Saurav and Conerly, Tom and Askell, Amanda and Henighan, Tom and Drain, Dawn and Perez, Ethan and Schiefer, Nicholas and Hatfield-Dodds, Zac and DasSarma, Nova and Tran-Johnson, Eli and others},
  journal = {arXiv preprint arXiv:2207.05221},
  year = {2022},
  url = {https://arxiv.org/abs/2207.05221},
}

@inproceedings{kuhn2023semantic,
  title = {Semantic Uncertainty: Linguistic Invariances for Uncertainty Estimation in Natural Language Generation},
  author = {Kuhn, Lorenz and Gal, Yarin and Farquhar, Sebastian},
  booktitle = {International Conference on Learning Representations (ICLR)},
  year = {2023},
  note = {arXiv:2302.09664},
  url = {https://arxiv.org/abs/2302.09664},
}

@article{farquhar2024semantic,
  title = {Detecting hallucinations in large language models using semantic entropy},
  author = {Farquhar, Sebastian and Kossen, Jannik and Kuhn, Lorenz and Gal, Yarin},
  journal = {Nature},
  volume = {630},
  number = {8017},
  pages = {625--630},
  year = {2024},
  doi = {10.1038/s41586-024-07421-0},
  url = {https://www.nature.com/articles/s41586-024-07421-0},
}

@inproceedings{wang2023selfconsistency,
  title = {Self-Consistency Improves Chain of Thought Reasoning in Language Models},
  author = {Wang, Xuezhi and Wei, Jason and Schuurmans, Dale and Le, Quoc and Chi, Ed and Narang, Sharan and Chowdhery, Aakanksha and Zhou, Denny},
  booktitle = {International Conference on Learning Representations (ICLR)},
  year = {2023},
  note = {arXiv:2203.11171},
  url = {https://arxiv.org/abs/2203.11171},
}

@inproceedings{manakul2023selfcheckgpt,
  title = {{SelfCheckGPT}: Zero-Resource Black-Box Hallucination Detection for Generative Large Language Models},
  author = {Manakul, Potsawee and Liusie, Adian and Gales, Mark J. F.},
  booktitle = {Proceedings of the 2023 Conference on Empirical Methods in Natural Language Processing (EMNLP)},
  year = {2023},
  note = {arXiv:2303.08896},
  url = {https://arxiv.org/abs/2303.08896},
}

@inproceedings{xiong2024llms,
  title = {Can {LLMs} Express Their Uncertainty? An Empirical Evaluation of Confidence Elicitation in {LLMs}},
  author = {Xiong, Miao and Hu, Zhiyuan and Lu, Xinyang and Li, Yifei and Fu, Jie and He, Junxian and Hooi, Bryan},
  booktitle = {International Conference on Learning Representations (ICLR)},
  year = {2024},
  note = {arXiv:2306.13063},
  url = {https://arxiv.org/abs/2306.13063},
}

@inproceedings{chen2023quantifying,
  title = {Quantifying Uncertainty in Answers from any Language Model and Enhancing their Trustworthiness},
  author = {Chen, Jiuhai and Mueller, Jonas},
  booktitle = {Proceedings of the 62nd Annual Meeting of the Association for Computational Linguistics (ACL)},
  year = {2024},
  note = {arXiv:2308.16175; BSDetector},
  url = {https://arxiv.org/abs/2308.16175},
}

@article{lin2022teaching,
  title = {Teaching Models to Express Their Uncertainty in Words},
  author = {Lin, Stephanie and Hilton, Jacob and Evans, Owain},
  journal = {Transactions on Machine Learning Research (TMLR)},
  year = {2022},
  note = {arXiv:2205.14334},
  url = {https://arxiv.org/abs/2205.14334},
}

@inproceedings{tian2023just,
  title = {Just Ask for Calibration: Strategies for Eliciting Calibrated Confidence Scores from Language Models Fine-Tuned with Human Feedback},
  author = {Tian, Katherine and Mitchell, Eric and Zhou, Allan and Sharma, Archit and Rafailov, Rafael and Yao, Huaxiu and Finn, Chelsea and Manning, Christopher D.},
  booktitle = {Proceedings of the 2023 Conference on Empirical Methods in Natural Language Processing},
  year = {2023},
  note = {arXiv:2305.14975},
  url = {https://aclanthology.org/2023.emnlp-main.330/},
  pages = {5433--5442},
  doi = {10.18653/v1/2023.emnlp-main.330},
  publisher = {Association for Computational Linguistics},
  address = {Singapore},
}

@article{kossen2024semantic,
  title = {Semantic Entropy Probes: Robust and Cheap Hallucination Detection in {LLMs}},
  author = {Kossen, Jannik and Han, Jiatong and Razzak, Muhammed and Schut, Lisa and Malik, Shreshth and Gal, Yarin},
  journal = {arXiv preprint arXiv:2406.15927},
  year = {2024},
  url = {https://arxiv.org/abs/2406.15927},
}

@article{zhang2025reasoning,
  title = {Reasoning Models Know When They're Right: Probing Hidden States for Self-Verification},
  author = {Zhang, Anqi and Chen, Yulin and Pan, Jane and Zhao, Chen and Panda, Aurojit and Li, Jinyang and He, He},
  journal = {arXiv preprint arXiv:2504.05419},
  year = {2025},
  url = {https://arxiv.org/abs/2504.05419},
}

@article{shrivastava2023llamas,
  title = {Llamas Know What {GPTs} Don't Show: Surrogate Models for Confidence Estimation},
  author = {Shrivastava, Vaishnavi and Liang, Percy and Kumar, Ananya},
  journal = {arXiv preprint arXiv:2311.08877},
  year = {2023},
  url = {https://arxiv.org/abs/2311.08877},
}

@inproceedings{schulman2016gae,
  title = {High-Dimensional Continuous Control Using Generalized Advantage Estimation},
  author = {Schulman, John and Moritz, Philipp and Levine, Sergey and Jordan, Michael and Abbeel, Pieter},
  booktitle = {International Conference on Learning Representations (ICLR)},
  year = {2016},
  note = {arXiv:1506.02438},
  url = {https://arxiv.org/abs/1506.02438},
}

@article{schulman2017ppo,
  title = {Proximal Policy Optimization Algorithms},
  author = {Schulman, John and Wolski, Filip and Dhariwal, Prafulla and Radford, Alec and Klimov, Oleg},
  journal = {arXiv preprint arXiv:1707.06347},
  year = {2017},
  url = {https://arxiv.org/abs/1707.06347},
}

@inproceedings{farebrother2024stop,
  title = {Stop Regressing: Training Value Functions via Classification for Scalable Deep {RL}},
  author = {Farebrother, Jesse and Orbay, Jordi and Vuong, Quan and Ali Taiga, Adrien and Chebotar, Yevgen and Xiao, Ted and Irpan, Alex and Levine, Sergey and Castro, Pablo Samuel and Faust, Aleksandra and Kumar, Aviral and Agarwal, Rishabh},
  booktitle = {International Conference on Machine Learning (ICML)},
  volume = {235},
  pages = {13049--13071},
  year = {2024},
  note = {arXiv:2403.03950},
  url = {https://proceedings.mlr.press/v235/farebrother24a.html},
}

@inproceedings{ouyang2022training,
  title = {Training Language Models to Follow Instructions with Human Feedback},
  author = {Ouyang, Long and Wu, Jeff and Jiang, Xu and Almeida, Diogo and Wainwright, Carroll L. and Mishkin, Pamela and Zhang, Chong and Agarwal, Sandhini and Slama, Katarina and Ray, Alex and others},
  booktitle = {Advances in Neural Information Processing Systems (NeurIPS)},
  year = {2022},
  note = {InstructGPT; arXiv:2203.02155},
  url = {https://arxiv.org/abs/2203.02155},
}

@article{ziegler2019finetuning,
  title = {Fine-Tuning Language Models from Human Preferences},
  author = {Ziegler, Daniel M. and Stiennon, Nisan and Wu, Jeffrey and Brown, Tom B. and Radford, Alec and Amodei, Dario and Christiano, Paul and Irving, Geoffrey},
  journal = {arXiv preprint arXiv:1909.08593},
  year = {2019},
  url = {https://arxiv.org/abs/1909.08593},
}

@article{shao2024deepseekmath,
  title = {{DeepSeekMath}: Pushing the Limits of Mathematical Reasoning in Open Language Models},
  author = {Shao, Zhihong and Wang, Peiyi and Zhu, Qihao and Xu, Runxin and Song, Junxiao and Bi, Xiao and Zhang, Haowei and Zhang, Mingchuan and Li, Y. K. and Wu, Y. and Guo, Daya},
  journal = {arXiv preprint arXiv:2402.03300},
  year = {2024},
  note = {GRPO},
  url = {https://arxiv.org/abs/2402.03300},
}

@article{deepseekai2025r1,
  title = {{DeepSeek-R1}: Incentivizing Reasoning Capability in {LLMs} via Reinforcement Learning},
  author = {{DeepSeek-AI}},
  journal = {arXiv preprint arXiv:2501.12948},
  year = {2025},
  url = {https://arxiv.org/abs/2501.12948},
}

@inproceedings{ahmadian2024basics,
  title = {Back to Basics: Revisiting {REINFORCE}-Style Optimization for Learning from Human Feedback in {LLMs}},
  author = {Ahmadian, Arash and Cremer, Chris and Gall{\'e}, Matthias and Fadaee, Marzieh and Kreutzer, Julia and Pietquin, Olivier and {\"U}st{\"u}n, Ahmet and Hooker, Sara},
  booktitle = {Proceedings of the 62nd Annual Meeting of the Association for Computational Linguistics (ACL)},
  pages = {12248--12267},
  year = {2024},
  note = {RLOO; arXiv:2402.14740},
  url = {https://arxiv.org/abs/2402.14740},
}

@inproceedings{kazemnejad2024vineppo,
  title = {{VinePPO}: Refining Credit Assignment in {RL} Training of {LLMs}},
  author = {Kazemnejad, Amirhossein and Aghajohari, Milad and Portelance, Eva and Sordoni, Alessandro and Reddy, Siva and Courville, Aaron and Le Roux, Nicolas},
  booktitle = {International Conference on Machine Learning (ICML)},
  year = {2025},
  note = {arXiv:2410.01679},
  url = {https://arxiv.org/abs/2410.01679},
}

@article{yuan2025vcppo,
  title = {What's Behind {PPO}'s Collapse in Long-{CoT}? Value Optimization Holds the Secret},
  author = {Yuan, Yufeng and Yue, Yu and Zhu, Ruofei and Fan, Tiantian and Yan, Lin},
  journal = {arXiv preprint arXiv:2503.01491},
  year = {2025},
  note = {VC-PPO},
  url = {https://arxiv.org/abs/2503.01491},
}

@article{yue2025vapo,
  title = {{VAPO}: Efficient and Reliable Reinforcement Learning for Advanced Reasoning Tasks},
  author = {Yue, Yu and Yuan, Yufeng and Yu, Qiying and Zuo, Xiaochen and Zhu, Ruofei and Xu, Wenyuan and Chen, Jiaze and Wang, Chengyi and Fan, Tiantian and others},
  journal = {arXiv preprint arXiv:2504.05118},
  year = {2025},
  url = {https://arxiv.org/abs/2504.05118},
}

@inproceedings{hu2025openreasoner,
  title = {Open-Reasoner-Zero: An Open Source Approach to Scaling Up Reinforcement Learning on the Base Model},
  author = {Hu, Jingcheng and Zhang, Yinmin and Han, Qi and Jiang, Daxin and Zhang, Xiangyu and Shum, Heung-Yeung},
  booktitle = {Advances in Neural Information Processing Systems (NeurIPS)},
  year = {2025},
  note = {arXiv:2503.24290},
  url = {https://arxiv.org/abs/2503.24290},
}

@inproceedings{liu2023dont,
  title = {Don't Throw Away Your Value Model! Generating More Preferable Text with Value-Guided Monte-Carlo Tree Search Decoding},
  author = {Liu, Jiacheng and Cohen, Andrew and Pasunuru, Ramakanth and Choi, Yejin and Hajishirzi, Hannaneh and Celikyilmaz, Asli},
  booktitle = {Conference on Language Modeling (COLM)},
  year = {2024},
  note = {PPO-MCTS; arXiv:2309.15028},
  url = {https://arxiv.org/abs/2309.15028},
}

@inproceedings{wang2024mathshepherd,
  title = {{Math-Shepherd}: Verify and Reinforce {LLMs} Step-by-Step Without Human Annotations},
  author = {Wang, Peiyi and Li, Lei and Shao, Zhihong and Xu, R. X. and Dai, Damai and Li, Yifei and Chen, Deli and Wu, Y. and Sui, Zhifang},
  booktitle = {Proceedings of the 62nd Annual Meeting of the Association for Computational Linguistics (ACL)},
  year = {2024},
  note = {arXiv:2312.08935},
  url = {https://arxiv.org/abs/2312.08935},
}

@inproceedings{guo2017calibration,
  title = {On Calibration of Modern Neural Networks},
  author = {Guo, Chuan and Pleiss, Geoff and Sun, Yu and Weinberger, Kilian Q.},
  booktitle = {Proceedings of the 34th International Conference on Machine Learning},
  series = {Proceedings of Machine Learning Research},
  volume = {70},
  pages = {1321--1330},
  year = {2017},
  publisher = {PMLR},
  url = {https://proceedings.mlr.press/v70/guo17a.html},
}

@article{jiang-etal-2021-know,
  title = {How Can We Know When Language Models Know? On the Calibration of Language Models for Question Answering},
  author = {Jiang, Zhengbao and Araki, Jun and Ding, Haibo and Neubig, Graham},
  journal = {Transactions of the Association for Computational Linguistics},
  volume = {9},
  pages = {962--977},
  year = {2021},
  doi = {10.1162/tacl_a_00407},
  url = {https://aclanthology.org/2021.tacl-1.57/},
}

@article{mielke-etal-2022-reducing,
  title = {Reducing Conversational Agents' Overconfidence Through Linguistic Calibration},
  author = {Mielke, Sabrina J. and Szlam, Arthur and Dinan, Emily and Boureau, Y-Lan},
  journal = {Transactions of the Association for Computational Linguistics},
  volume = {10},
  pages = {857--872},
  year = {2022},
  doi = {10.1162/tacl_a_00494},
  url = {https://aclanthology.org/2022.tacl-1.50/},
}

@book{suttonbarto2018,
  title = {Reinforcement Learning: An Introduction},
  author = {Sutton, Richard S. and Barto, Andrew G.},
  edition = {2},
  year = {2018},
  publisher = {The MIT Press},
  address = {Cambridge, MA},
  isbn = {9780262039246},
  url = {https://mitpress.mit.edu/9780262039246/reinforcement-learning/},
}

@inproceedings{zhang-etal-2024-r,
  title = {{R}-Tuning: Instructing Large Language Models to Say `I Don't Know'},
  author = {Zhang, Hanning and Diao, Shizhe and Lin, Yong and Fung, Yi and Lian, Qing and Wang, Xingyao and Chen, Yangyi and Ji, Heng and Zhang, Tong},
  booktitle = {Proceedings of the 2024 Conference of the North American Chapter of the Association for Computational Linguistics: Human Language Technologies (Volume 1: Long Papers)},
  year = {2024},
  pages = {7113--7139},
  address = {Mexico City, Mexico},
  publisher = {Association for Computational Linguistics},
  doi = {10.18653/v1/2024.naacl-long.394},
  url = {https://aclanthology.org/2024.naacl-long.394/},
}

@inproceedings{xu-etal-2024-sayself,
  title = {{SaySelf}: Teaching {LLM}s to Express Confidence with Self-Reflective Rationales},
  author = {Xu, Tianyang and Wu, Shujin and Diao, Shizhe and Liu, Xiaoze and Wang, Xingyao and Chen, Yangyi and Gao, Jing},
  booktitle = {Proceedings of the 2024 Conference on Empirical Methods in Natural Language Processing},
  year = {2024},
  publisher = {Association for Computational Linguistics},
  address = {Miami, Florida, USA},
  url = {https://aclanthology.org/2024.emnlp-main.343/},
}

@article{silver2016alphago,
  title = {Mastering the game of {Go} with deep neural networks and tree search},
  author = {Silver, David and Huang, Aja and Maddison, Chris J. and Guez, Arthur and Sifre, Laurent and van den Driessche, George and Schrittwieser, Julian and Antonoglou, Ioannis and Panneershelvam, Veda and Lanctot, Marc and others},
  journal = {Nature},
  volume = {529},
  number = {7587},
  pages = {484--489},
  year = {2016},
  publisher = {Nature Publishing Group},
  doi = {10.1038/nature16961},
  url = {https://www.nature.com/articles/nature16961},
}

@inproceedings{yu2024ovm,
  title = {{OVM}, Outcome-supervised Value Models for Planning in Mathematical Reasoning},
  author = {Yu, Fei and Gao, Anningzhe and Wang, Benyou},
  booktitle = {Findings of the Association for Computational Linguistics: NAACL 2024},
  pages = {858--875},
  year = {2024},
  publisher = {Association for Computational Linguistics},
  doi = {10.18653/v1/2024.findings-naacl.55},
  note = {arXiv:2311.09724},
  url = {https://aclanthology.org/2024.findings-naacl.55/},
}

@inproceedings{zhou2024archer,
  title = {{ArCHer}: Training Language Model Agents via Hierarchical Multi-Turn {RL}},
  author = {Zhou, Yifei and Zanette, Andrea and Pan, Jiayi and Levine, Sergey and Kumar, Aviral},
  booktitle = {Proceedings of the 41st International Conference on Machine Learning (ICML)},
  series = {PMLR},
  volume = {235},
  pages = {62178--62209},
  year = {2024},
  note = {arXiv:2402.19446},
  url = {https://proceedings.mlr.press/v235/zhou24t.html},
}

@article{cobbe2021training,
  title = {Training Verifiers to Solve Math Word Problems},
  author = {Cobbe, Karl and Kosaraju, Vineet and Bavarian, Mohammad and Chen, Mark and Jun, Heewoo and Kaiser, Lukasz and Plappert, Matthias and Tworek, Jerry and Hilton, Jacob and Nakano, Reiichiro and Hesse, Christopher and Schulman, John},
  journal = {arXiv preprint arXiv:2110.14168},
  year = {2021},
  note = {GSM8K},
  url = {https://arxiv.org/abs/2110.14168},
}

@article{wu2025mitigating,
  title = {Mitigating {LLM} Hallucination via Behaviorally Calibrated Reinforcement Learning},
  author = {Wu, Jiayun and Liu, Jiashuo and Zeng, Zhiyuan and Zhan, Tianyang and Cai, Tianle and Huang, Wenhao},
  journal = {arXiv preprint arXiv:2512.19920},
  year = {2025},
  doi = {10.48550/arXiv.2512.19920},
  url = {https://arxiv.org/abs/2512.19920},
}

@article{grattafiori2024llama3,
  title = {The {Llama} 3 Herd of Models},
  author = {Grattafiori, Aaron and Dubey, Abhimanyu and Jauhri, Abhinav and Pandey, Abhinav and Kadian, Abhishek and others},
  journal = {arXiv preprint arXiv:2407.21783},
  year = {2024},
  url = {https://arxiv.org/abs/2407.21783},
}

@article{qwen2025qwen25,
  title = {{Qwen2.5} Technical Report},
  author = {{Qwen Team}},
  journal = {arXiv preprint arXiv:2412.15115},
  year = {2024},
  url = {https://arxiv.org/abs/2412.15115},
}

@article{jiang2023mistral,
  title = {{Mistral} {7B}},
  author = {Jiang, Albert Q. and Sablayrolles, Alexandre and Mensch, Arthur and Bamford, Chris and Chaplot, Devendra Singh and de las Casas, Diego and Bressand, Florian and Lengyel, Gianna and Lample, Guillaume and Saulnier, Lucile and others},
  journal = {arXiv preprint arXiv:2310.06825},
  year = {2023},
  url = {https://arxiv.org/abs/2310.06825},
}

@misc{numinamath2024,
  title = {{NuminaMath}},
  author = {Li, Jia and Beeching, Edward and Tunstall, Lewis and Lipkin, Ben and Soletskyi, Roman and Huang, Shengyi Costa and Rasul, Kashif and Yu, Longhui and Jiang, Albert and Shen, Ziju and Qin, Zihan and Dong, Bin and Zhou, Li and Fleureau, Yann and Lample, Guillaume and Polu, Stanislas},
  year = {2024},
  publisher = {Numina},
  howpublished = {\url{https://huggingface.co/datasets/AI-MO/NuminaMath-CoT}},
  url = {https://huggingface.co/datasets/AI-MO/NuminaMath-CoT},
}

@inproceedings{hendrycks2021apps,
  title = {Measuring Coding Challenge Competence With {APPS}},
  author = {Hendrycks, Dan and Basart, Steven and Kadavath, Saurav and Mazeika, Mantas and Arora, Akul and Guo, Ethan and Burns, Collin and Puranik, Samir and He, Horace and Song, Dawn and Steinhardt, Jacob},
  booktitle = {Advances in Neural Information Processing Systems (NeurIPS) Datasets and Benchmarks Track},
  year = {2021},
  note = {arXiv:2105.09938},
  url = {https://arxiv.org/abs/2105.09938},
}

@article{li2023taco,
  title = {{TACO}: Topics in Algorithmic {COde} Generation Dataset},
  author = {Li, Rongao and Fu, Jie and Zhang, Bo-Wen and Huang, Tao and Sun, Zhihong and Lyu, Chen and Liu, Guang and Jin, Zhi and Li, Ge},
  journal = {arXiv preprint arXiv:2312.14852},
  year = {2023},
  url = {https://arxiv.org/abs/2312.14852},
}

@article{clark2018arc,
  title = {Think you have Solved Question Answering? Try {ARC}, the {AI2} Reasoning Challenge},
  author = {Clark, Peter and Cowhey, Isaac and Etzioni, Oren and Khot, Tushar and Sabharwal, Ashish and Schoenick, Carissa and Tafjord, Oyvind},
  journal = {arXiv preprint arXiv:1803.05457},
  year = {2018},
  url = {https://arxiv.org/abs/1803.05457},
}

@inproceedings{hu2022lora,
  title = {{LoRA}: Low-Rank Adaptation of Large Language Models},
  author = {Hu, Edward J. and Shen, Yelong and Wallis, Phillip and Allen-Zhu, Zeyuan and Li, Yuanzhi and Wang, Shean and Wang, Lu and Chen, Weizhu},
  booktitle = {International Conference on Learning Representations (ICLR)},
  year = {2022},
  note = {arXiv:2106.09685},
  url = {https://arxiv.org/abs/2106.09685},
}

@article{gneiting2007strictly,
  title = {Strictly proper scoring rules, prediction, and estimation},
  author = {Gneiting, Tilmann and Raftery, Adrian E.},
  journal = {Journal of the American Statistical Association},
  volume = {102},
  number = {477},
  pages = {359--378},
  year = {2007},
  url = {https://www.stat.washington.edu/raftery/Research/PDF/Gneiting2007jasa.pdf},
}

@article{clemencon2008ranking,
  title = {Ranking and empirical minimization of {U}-statistics},
  author = {Cl{\'e}men{\c{c}}on, St{\'e}phan and Lugosi, G{\'a}bor and Vayatis, Nicolas},
  journal = {The Annals of Statistics},
  volume = {36},
  number = {2},
  pages = {844--874},
  year = {2008},
  url = {https://projecteuclid.org/journals/annals-of-statistics/volume-36/issue-2/Ranking-and-Empirical-Minimization-of-U-statistics/10.1214/009052607000000910.full},
}

@misc{codeforces_data,
  title = {CodeForces},
  author = {Guilherme Penedo and Anton Lozhkov and Hynek Kydlíček and Loubna Ben Allal and Edward Beeching and Agustín Piqueres Lajarín and Quentin Gallouédec and Nathan Habib and Lewis Tunstall and Leandro von Werra},
  year = {2025},
  publisher = {Hugging Face},
  journal = {Hugging Face repository},
  howpublished = {\url{https://huggingface.co/datasets/open-r1/codeforces}},
  url = {https://huggingface.co/datasets/open-r1/codeforces},
}

@article{qwen3,
  title = {Qwen3 Technical Report},
  author = {Yang, An and Li, Anfeng and Yang, Baosong and Zhang, Beichen and Hui, Binyuan and Zheng, Bo and Yu, Bowen and Gao, Chang and Huang, Chengen and Lv, Chenxu and others},
  journal = {arXiv preprint arXiv:2505.09388},
  year = {2025},
  url = {https://arxiv.org/abs/2505.09388},
}

@article{nemotroncascade,
  title = {{Nemotron-Cascade}: Scaling Cascaded Reinforcement Learning for General-Purpose Reasoning Models},
  author = {Wang, Boxin and Lee, Chankyu and Lee, Nayeon and Lin, Sheng-Chieh and Dai, Wenliang and Chen, Yang and Chen, Yangyi and Yang, Zhuolin and Liu, Zihan and Shoeybi, Mohammad and Catanzaro, Bryan and Ping, Wei},
  journal = {arXiv preprint arXiv:2512.13607},
  year = {2025},
  url = {https://arxiv.org/abs/2512.13607},
}

@article{li2022alphacode,
  title = {Competition-Level Code Generation with {AlphaCode}},
  author = {Li, Yujia and Choi, David and Chung, Junyoung and Kushman, Nate and Schrittwieser, Julian and Leblond, R{\'e}mi and Eccles, Tom and Keeling, James and Gimeno, Felix and Dal Lago, Agustin and Hubert, Thomas and Choy, Peter and de Masson d'Autume, Cyprien and Babuschkin, Igor and Chen, Xinyun and Huang, Po-Sen and Welbl, Johannes and Gowal, Sven and Cherepanov, Alexey and Molloy, James and Mankowitz, Daniel J. and Sutherland Robson, Esme and Kohli, Pushmeet and de Freitas, Nando and Kavukcuoglu, Koray and Vinyals, Oriol},
  journal = {Science},
  volume = {378},
  number = {6624},
  pages = {1092--1097},
  year = {2022},
  note = {CodeContests},
  url = {https://arxiv.org/abs/2203.07814},
}

@inproceedings{jain2024livecodebench,
  title = {{LiveCodeBench}: Holistic and Contamination Free Evaluation of Large Language Models for Code},
  author = {Jain, Naman and Han, King and Gu, Alex and Li, Wen-Ding and Yan, Fanjia and Zhang, Tianjun and Wang, Sida and Solar-Lezama, Armando and Sen, Koushik and Stoica, Ion},
  booktitle = {International Conference on Learning Representations (ICLR)},
  year = {2025},
  note = {arXiv:2403.07974},
  url = {https://arxiv.org/abs/2403.07974},
}

@article{qian2025toolrl,
  title = {{ToolRL}: Reward is All Tool Learning Needs},
  author = {Qian, Cheng and Acikgoz, Emre Can and He, Qi and Wang, Hongru and Chen, Xiusi and Hakkani-T{\"u}r, Dilek and Tur, Gokhan and Ji, Heng},
  journal = {arXiv preprint arXiv:2504.13958},
  year = {2025},
  url = {https://arxiv.org/abs/2504.13958},
}

@inproceedings{li2023apibank,
  title = {{API-Bank}: A Comprehensive Benchmark for Tool-Augmented {LLMs}},
  author = {Li, Minghao and Zhao, Yingxiu and Yu, Bowen and Song, Feifan and Li, Hangyu and Yu, Haiyang and Li, Zhoujun and Huang, Fei and Li, Yongbin},
  booktitle = {Proceedings of the 2023 Conference on Empirical Methods in Natural Language Processing (EMNLP)},
  pages = {3102--3116},
  year = {2023},
  url = {https://aclanthology.org/2023.emnlp-main.187/},
}

@article{ma2025generalreasoner,
  title = {{General-Reasoner}: Advancing {LLM} Reasoning Across All Domains},
  author = {Ma, Xueguang and Liu, Qian and Jiang, Dongfu and Zhang, Ge and Ma, Zejun and Chen, Wenhu},
  journal = {arXiv preprint arXiv:2505.14652},
  year = {2025},
  note = {WebInstruct-verified},
  url = {https://arxiv.org/abs/2505.14652},
}

@inproceedings{leng2025taming,
  title = {Taming Overconfidence in {LLMs}: Reward Calibration in {RLHF}},
  author = {Leng, Jixuan and Huang, Chengsong and Zhu, Banghua and Huang, Jiaxin},
  booktitle = {International Conference on Learning Representations (ICLR)},
  year = {2025},
  note = {arXiv:2410.09724},
  url = {https://arxiv.org/abs/2410.09724},
}

@inproceedings{azaria2023internal,
  title = {The Internal State of an {LLM} Knows When It's Lying},
  author = {Azaria, Amos and Mitchell, Tom},
  booktitle = {Findings of the Association for Computational Linguistics: EMNLP 2023},
  pages = {967--976},
  year = {2023},
  url = {https://aclanthology.org/2023.findings-emnlp.68/},
}

@inproceedings{pal2022medmcqa,
  title = {{MedMCQA}: A Large-scale Multi-Subject Multi-Choice Dataset for Medical Domain Question Answering},
  author = {Pal, Ankit and Umapathi, Logesh Kumar and Sankarasubbu, Malaikannan},
  booktitle = {Proceedings of the Conference on Health, Inference, and Learning (CHIL)},
  series = {Proceedings of Machine Learning Research},
  volume = {174},
  pages = {248--260},
  publisher = {PMLR},
  year = {2022},
  url = {https://proceedings.mlr.press/v174/pal22a.html},
}

@inproceedings{jin2019pubmedqa,
  title = {{PubMedQA}: A Dataset for Biomedical Research Question Answering},
  author = {Jin, Qiao and Dhingra, Bhuwan and Liu, Zhengping and Cohen, William and Lu, Xinghua},
  booktitle = {Proceedings of the 2019 Conference on Empirical Methods in Natural Language Processing and the 9th International Joint Conference on Natural Language Processing (EMNLP-IJCNLP)},
  pages = {2567--2577},
  publisher = {Association for Computational Linguistics},
  year = {2019},
  doi = {10.18653/v1/D19-1259},
  url = {https://aclanthology.org/D19-1259/},
}

@article{zhang2024chemllm,
  title = {{ChemLLM}: A Chemical Large Language Model},
  author = {Zhang, Di and Liu, Wei and Tan, Qian and Chen, Jingdan and Yan, Hang and Yan, Yuliang and Li, Jiatong and Huang, Weiran and Yue, Xiangyu and Ouyang, Wanli and Zhou, Dongzhan and Zhang, Shufei and Su, Mao and Zhong, Han-Sen and Li, Yuqiang},
  journal = {arXiv preprint arXiv:2402.06852},
  year = {2024},
  url = {https://arxiv.org/abs/2402.06852},
}

@article{wu2026teaching,
  title = {Teaching llms to self-evolve: Cultivating core meta-skills with reinforcement learning},
  author = {Wu, Shujin and Qian, Cheng and Chen, Xiusi and Ji, Heng},
  journal = {arXiv preprint arXiv:2607.21971},
  year = {2026},
  url = {https://arxiv.org/abs/2607.21971},
}

@inproceedings{xu2025ugphysics,
  title = {{UGPhysics}: A Comprehensive Benchmark for Undergraduate Physics Reasoning with Large Language Models},
  author = {Xu, Xin and Xu, Qiyun and Xiao, Tong and Chen, Tianhao and Yan, Yuchen and Zhang, Jiaxin and Diao, Shizhe and Yang, Can and Wang, Yang},
  booktitle = {Proceedings of the 42nd International Conference on Machine Learning},
  series = {Proceedings of Machine Learning Research},
  volume = {267},
  pages = {69849--69877},
  publisher = {PMLR},
  year = {2025},
  url = {https://proceedings.mlr.press/v267/xu25ai.html},
}

@inproceedings{ulmer2024calibrating,
  title = {Calibrating Large Language Models Using Their Generations Only},
  author = {Ulmer, Dennis and Gubri, Martin and Lee, Hwaran and Yun, Sangdoo and Oh, Seong},
  booktitle = {Proceedings of the 62nd Annual Meeting of the Association for Computational Linguistics (Volume 1: Long Papers)},
  year = {2024},
  month = {aug},
  address = {Bangkok, Thailand},
  publisher = {Association for Computational Linguistics},
  pages = {15440--15459},
  doi = {10.18653/v1/2024.acl-long.824},
  url = {https://aclanthology.org/2024.acl-long.824/},
}

@inproceedings{kapoor2024taught,
  title = {Large Language Models Must Be Taught to Know What They Don't Know},
  author = {Kapoor, Sanyam and Gruver, Nate and Roberts, Manley and Collins, Katherine and Pal, Arka and Bhatt, Umang and Weller, Adrian and Dooley, Samuel and Goldblum, Micah and Wilson, Andrew Gordon},
  booktitle = {Advances in Neural Information Processing Systems},
  volume = {37},
  year = {2024},
  doi = {10.52202/079017-2729},
  url = {https://proceedings.neurips.cc/paper_files/paper/2024/hash/9c20f16b05f5e5e70fa07e2a4364b80e-Abstract-Conference.html},
}

@misc{nvidia2025openscience2,
  title = {{OpenScienceReasoning-2}},
  author = {{NVIDIA}},
  year = {2025},
  howpublished = {\url{https://huggingface.co/datasets/nvidia/OpenScienceReasoning-2}},
  note = {Hugging Face dataset, CC BY 4.0},
  url = {https://huggingface.co/datasets/nvidia/OpenScienceReasoning-2/blob/main/README.md},
}

@misc{aime,
  title = {{AIME} 2024 and {AIME} 2025},
  author = {{Mathematical Association of America}},
  howpublished = {American Invitational Mathematics Examination},
  year = {2025}
}

@inproceedings{naeini2015ece,
  title = {Obtaining Well Calibrated Probabilities Using {Bayesian} Binning},
  author = {Naeini, Mahdi Pakdaman and Cooper, Gregory F. and Hauskrecht, Milos},
  booktitle = {Proceedings of the AAAI Conference on Artificial Intelligence},
  year = {2015},
  url = {https://ojs.aaai.org/index.php/AAAI/article/view/9602},
}

@article{brier1950,
  title = {Verification of forecasts expressed in terms of probability},
  author = {Brier, Glenn W.},
  journal = {Monthly Weather Review},
  volume = {78},
  number = {1},
  pages = {1--3},
  year = {1950},
  url = {https://journals.ametsoc.org/view/journals/mwre/78/1/1520-0493_1950_078_0001_vofeit_2_0_co_2.xml},
}

@article{dhaliwal2026critic,
  title = {From Critic to Confidence: {PPO} for Language-Based Quantitative Prediction with Confidence Estimation},
  author = {Dhaliwal, Mehak and Tadayon, Rasta and Hua, Andong and Jeong, Haewon and Qin, Yao},
  journal = {arXiv preprint arXiv:2607.12687},
  year = {2026},
  doi = {10.48550/arXiv.2607.12687},
  url = {https://arxiv.org/abs/2607.12687},
}

\newpage
\appendix
\section*{Appendix}

% ============================================================
% Appendix: formal statements and proofs for Section 3 (condensed).
% Place after \appendix. Uses the theorem environments from the preamble
% (see the header of section3_main.tex). New bib entries at the end.
% ============================================================

\section{Formal Statements and Proofs for Section~\ref{sec:theory}}
\label{app:theory}

\setcounter{theorem}{0}
\renewcommand{\thetheorem}{A.\arabic{theorem}}

\subsection{Setting}
\label{app:setup}

A prompt $x \sim \mathcal{D}_X$ is answered token by token by $\pi_\theta$, producing $y = (y_1, \dots, y_T)$ with $T \le T_{\max}$. The state at step $t$ is the prefix $s_t = (x, y_{1:t})$, so $s_T = (x, y)$, and $s_t$ is determined by $s_u$ for every $u \ge t$. A grader assigns correctness $c \in \{0, 1\}$ to $(x, y)$; the grader may be stochastic, with a deterministic grader as a special case. 
The reward is terminal, $R(x, y) = c$, with no intermediate rewards and $\gamma = 1$, as in PPO-style training of LLMs. 
Here, \(\mathbb P_\pi\) is the distribution of sampled prompts, responses, and grader labels, and \(\mathbb E_\pi[\cdot]\) means averaging over these samples. For example, \(\mathbb E_\pi[c]\) is the average accuracy. 
We use “a.s.” to mean “with probability one.”
The value function \(V^\pi(s_t)\) is the expected final reward given prefix \(s_t\), and \(q(s_t):=\mathbb P_\pi(c=1\mid s_t)\) is the probability of a correct final answer. If a response ends at step \(T\), we set \(s_t=s_T\) for \(t>T\).

% $\mathbb{P}_\pi$ and $\mathbb{E}_\pi$ refer to the law of $(x, y, c)$ with $y \sim \pi_\theta(\cdot \mid x)$, and ``a.s.'' means $\mathbb{P}_\pi$-almost surely. 
% The value function $V^\pi(s_t)$ is the expected return from $s_t$, and we write $q(s_t) := \mathbb{P}_\pi(c = 1 \mid s_t)$. To make $s_t$ well defined for every $t \le T_{\max}$ on every trajectory, we extend trajectories past termination by the absorbing convention $s_u := s_T$ for $u > T$; then $s_t$ is determined by $s_u$ for all $u \ge t$, and a statement ``for fixed $t$'' below is a statement about the prefix of length $\min(t, T)$.

For a hypothesis class $\mathcal{G}_t$ of functions of $s_t$ and a loss $\ell$, the two objectives are
\begin{align}
  \mathcal{L}^{\mathrm{conf}}_t(g) &= \mathbb{E}_\pi\big[\ell(g(s_t), c)\big], \label{eq:app-conf}\\
  \mathcal{L}^{\mathrm{critic}}_t(V) &= \mathbb{E}_\pi\big[(V(s_t) - R)^2\big]. \label{eq:app-mc}
\end{align}
These are the calibration and critic objectives in Eqs.~\eqref{eq:calib} and~\eqref{eq:critic}, evaluated at any prefix \(s_t\).
In practice, the critic sums Eq.~\eqref{eq:app-mc} over token positions $t$, using the same final reward $c$ as the target at every position. 
As shown below, the expected squared error at each position is minimized by predicting $q(s_t)$. Thus, a critic that can represent $q$ also minimizes the sum over positions.
% We use a standard fact about strictly proper losses: if $c\sim\mathrm{Bernoulli}(q)$, then predicting $p=q$ uniquely minimizes the expected loss $\mathbb E[\ell(p,c)]$ \citep{gneiting2007strictly}.
We use a standard fact about strictly proper losses: when a binary label $c$ equals $1$ with probability $q$ (and $0$ otherwise), predicting $p=q$ uniquely minimizes the expected loss $\mathbb E[\ell(p,c)]$ \citep{gneiting2007strictly}.
% We use one standard fact about strictly proper scoring rules \citep{gneiting2007strictly}: if $c \sim \mathrm{Bernoulli}(q)$, then $\mathbb{E}[\ell(p, c)] - \mathbb{E}[\ell(q, c)] = d_\ell(p, q) \ge 0$, with equality if and only if $p = q$, where $d_\ell(p, q) = (p - q)^2$ for the squared loss and $d_\ell(p, q) = \mathrm{KL}(q \,\|\, p)$ for the log loss.
\subsection{Main Theorem}
\label{app:main}

\begin{theorem}[Value prediction and confidence calibration]
\label{thm:formal}
Assume the final reward is $R=c$, and fix $t\in\{0,\dots,T_{\max}\}$.
\begin{itemize}
  \item[(i)] The value at a prefix is its probability of leading to a correct answer:
  $V^\pi(s_t)=q(s_t)$ a.s. For any $t\le u\le T_{\max}$,
  \[
    \mathbb E_\pi[V^\pi(s_u)\mid s_t]=V^\pi(s_t).
  \]
  \item[(ii)] With squared loss,
  $\mathcal L_t^{\mathrm{conf}}(h)=\mathcal L_t^{\mathrm{critic}}(h)$
  for every $h\in\mathcal G_t$. The two objectives therefore have the
  same minimizers, including when computed on the same sample of responses.
  \item[(iii)] With any strictly proper loss, the expected confidence loss
  is uniquely minimized by $g(s_t)=q(s_t)=V^\pi(s_t)$ a.s. If $q$ belongs
  to $\mathcal G_t$, the calibrator and the squared-loss critic can both
  achieve this optimum.
\end{itemize}
\end{theorem}

\begin{proof}
For (i), the return from every prefix is the final reward $c$, so
$V^\pi(s_t)=\mathbb E_\pi[c\mid s_t]=q(s_t)$.
Fix $u\ge t$. Given $s_t$, the later prefix $s_u$ is random.
Let $z$ denote a possible value of $s_u$, and set
$p_z:=\mathbb P_\pi(s_u=z\mid s_t)$.
Because $c$ is either $0$ or $1$,
\[
\begin{aligned}
&\mathbb E_\pi[c\mid s_u=z]\\
&=0\cdot\mathbb P_\pi(c=0\mid s_u=z)\\
&\quad+1\cdot\mathbb P_\pi(c=1\mid s_u=z)\\
&=q(z).
\end{aligned}
\]
Thus the $c=0$ cases contribute zero. Since $s_u=z$ also
determines $s_t$, the probability $p_zq(z)$ equals
$\mathbb P_\pi(c=1,s_u=z\mid s_t)$. Therefore
\[
\begin{aligned}
&\mathbb E_\pi[V^\pi(s_u)\mid s_t]\\
&=\mathbb E_\pi[\mathbb E_\pi[c\mid s_u]\mid s_t]\\
&=\sum_z p_z\,\mathbb E_\pi[c\mid s_u=z]\\
&=\sum_z p_z\,q(z)\\
&=\sum_z\mathbb P_\pi(c=1,s_u=z\mid s_t)\\
&=\mathbb P_\pi(c=1\mid s_t)\\
&=V^\pi(s_t).
\end{aligned}
\]
For fixed $u$, the cases $s_u=z$ are disjoint and cover all responses. Summing the probabilities of the $c=1$ cases therefore gives $\mathbb P_\pi(c=1\mid s_t)$, which equals $V^\pi(s_t)$ because $c$ is binary.

For (ii), take the confidence loss to be squared error. 
Since $R=c$, for every
$h\in\mathcal G_t$,
\[
\begin{aligned}
\mathcal L_t^{\mathrm{critic}}(h)
&=\mathbb E_\pi[(h(s_t)-R)^2]\\
&=\mathbb E_\pi[(h(s_t)-c)^2]\\
&=\mathcal L_t^{\mathrm{conf}}(h).
\end{aligned}
\]
\looseness=-1
The equality also holds term by term on any shared sample,
since $R^{(i)}=c^{(i)}$ for every response.
Thus the two losses, and their minimizers, are identical.

For (iii), given $s_t$, the label $c$ equals $1$ with probability
$q(s_t)$. Strict propriety gives
\[
\begin{aligned}
&\mathbb E_\pi[\ell(g(s_t),c)\mid s_t]\\
&\quad\ge
  \mathbb E_\pi[\ell(q(s_t),c)\mid s_t],
\end{aligned}
\]
with equality only when $g(s_t)=q(s_t)$. Averaging over $s_t$ gives
$\mathcal L_t^{\mathrm{conf}}(g)\ge
\mathcal L_t^{\mathrm{conf}}(q)$, with equality only when $g=q$
a.s. By (i), $q=V^\pi$. If $q\in\mathcal G_t$, it also minimizes
the squared-loss critic objective by (ii).
\end{proof}

The theorem concerns the prediction target, regardless of estimator
architecture. Self-estimators and external estimators can both target
$q(s_t)$ when they predict from the same prefix.

\subsection{Bootstrapped Critics}
\label{app:td}

With a terminal reward, the $n$-step target from $s_t$ for a candidate value function $V$ is $V(s_{t+n})$ if $t + n \le T$ and $R$ otherwise, so the $\lambda$-return used by GAE \citep{schulman2016gae} is
\begin{align}
\widehat{G}^{\lambda}_t(V)
={}&(1-\lambda)
\Bigl[\sum_{n=1}^{T-t}
\lambda^{n-1}V(s_{t+n})\Bigr]\nonumber\\
&+\lambda^{T-t}R.
\label{eq:app-lambda}
\end{align}
$\lambda = 1$ gives the Monte Carlo target $R$, $\lambda = 0$ gives the one-step TD target $V(s_{t+1})$, and a terminal prefix always has target $R$. 
The target uses $V$'s own predictions at later prefixes. Let
$(\mathcal T^\lambda V)(s_t):=
\mathbb E_\pi[\widehat G_t^\lambda(V)\mid s_t]$
be the average target computed using $V$. We seek a value function
that agrees with this average target:
\[
V(s_t)=(\mathcal T^\lambda V)(s_t).
\]
Such a value function is called a fixed point.

Responses can end at different steps, so the number of value
predictions and the weight on $R$ in Eq.~\eqref{eq:app-lambda}
vary across rollouts. 
To use the same weights for every rollout, let $Z_n=V(s_{t+n})$ when $t+n\le T$, and let $Z_n=R$ afterward.
For $0\le\lambda<1$, the weights after termination add up to
the original weight on $R$:
\[
(1-\lambda)\sum_{n>T-t}\lambda^{n-1}
=\lambda^{T-t}.
\]
Thus the same target can be written as
\begin{equation}
  \widehat G^\lambda_t(V)
  =(1-\lambda)\sum_{n=1}^{\infty}\lambda^{n-1}Z_n.
  \label{eq:app-series}
\end{equation}

\begin{proposition}
\label{prop:td}
Under the assumptions of Theorem~\ref{thm:formal}, for every
$\lambda\in[0,1]$, $\mathcal T^\lambda$ has a unique fixed point
(up to null sets), namely $V^\pi=q$.
Moreover, starting from any bounded $V$, exact repeated updates
reach this fixed point:
\[
(\mathcal T^\lambda)^k V=V^\pi
\quad\text{a.s. for every }k\ge T_{\max}+1.
\]
\end{proposition}

\begin{proof}
\emph{Why $V^\pi$ is a fixed point.}
When $\lambda=1$, the target is $R$ regardless of $V$.
Hence, for any $V$,
\[
\begin{aligned}
(\mathcal T^1 V)(s_t)
&=\mathbb E_\pi[R\mid s_t]\\
&=V^\pi(s_t).
\end{aligned}
\]
In particular, $\mathcal T^1 V^\pi=V^\pi$.

Now take $\lambda<1$ and use Eq.~\eqref{eq:app-series} with
$V=V^\pi$. Each $Z_n$ uses the true value of a later prefix
until the response ends, and uses the observed reward $R$
afterward. With prefixes held at $s_T$ after termination,
\[
\begin{aligned}
Z_n
&=V^\pi(s_{t+n})\\
&\quad+
\bigl(R-V^\pi(s_T)\bigr)
\mathbf{1}\{t+n>T\}.
\end{aligned}
\]

\emph{Why the fixed point is unique.}
At a terminal prefix $s_T$, the target is $R$ and does not
depend on the starting function $V$. Thus one exact update
sets its value to $\mathbb E_\pi[R\mid s_T]=V^\pi(s_T)$.
At any earlier prefix, the target uses $V$ only at later
prefixes. Once those later values are correct, the next
update also sets the current value to $V^\pi(s_t)$.
We can therefore work backward from the end of the response.
After $k$ updates, all prefixes of length at least
$T_{\max}-k+1$ have the correct value.
After $T_{\max}+1$ updates, all prefixes do. 
The same backward argument forces any fixed point to equal $V^\pi$.
\end{proof}

% \subsection{Properties of the Shared Solution}
% \label{app:props}

% \begin{corollary}
% \label{cor:props}
% Let $q = V^\pi$ as in Theorem~\ref{thm:formal}. (a) $q$ is calibrated: $\mathbb{P}_\pi(c = 1 \mid q(s_t) = p) = p$. (b) $q$, and any strictly increasing transform of it, attains the maximal AUROC among all functions of $s_t$. (c) $\mathbb{E}_\pi[V^\pi(s_{t+1}) \mid s_t] = V^\pi(s_t)$.
% \end{corollary}

% \begin{proof}
% (a) Conditioning first on $s_t$ and then on the event $\{q(s_t) = p\}$, the tower property gives $\mathbb{P}_\pi(c = 1 \mid q(s_t) = p) = \mathbb{E}_\pi[q(s_t) \mid q(s_t) = p] = p$.
% (b) The posterior probability is the AUC-optimal scoring function for bipartite ranking \citep{clemencon2008ranking}: thresholding $q$ is a likelihood-ratio test, which is most powerful at every level by the Neyman--Pearson lemma \citep{neyman1933problem}, so its ROC curve dominates all others; strictly increasing transforms preserve rankings.
% (c) This is Theorem~\ref{thm:formal}(i) with $u = t + 1$.
% \end{proof}

\subsection{Properties of the Shared Target}
\label{app:props}

Theorem~\ref{thm:formal} shows that confidence estimation and
critic learning share the population target
$q(s_t)=V^\pi(s_t)$. The following properties explain why this
target is useful for calibration and ranking.

\begin{corollary}
\label{cor:props}
Let $Q=q(s_t)=V^\pi(s_t)$.
(a) $Q$ is calibrated: $\mathbb E_\pi[c\mid Q]=Q$ a.s.
(b) Among scores based on $s_t$, $q(s_t)$ achieves the highest
possible population AUROC, provided both correctness labels occur.
Any strictly increasing transformation of $q$ has the same AUROC.
\end{corollary}

\begin{proof}
For (a), $q(s_t)$ is the probability of a correct answer given
the prefix $s_t$. If we group together prefixes for which
$q(s_t)=p$, each prefix in that group has correctness probability
$p$. The group's average correctness probability is therefore
also $p$. Equivalently,
\[
\mathbb E_\pi[c\mid Q]
=\mathbb E_\pi[q(s_t)\mid Q]
=Q.
\]

For (b), consider two independently sampled prefixes $s$ and
$s'$ at the same step. Write $a=q(s)$ and $b=q(s')$.
Given these prefixes, the probability that the first response
is correct and the second is incorrect is $a(1-b)$.
The probability of the reverse outcome is $(1-a)b$.
Their difference is
\[
a(1-b)-(1-a)b=a-b.
\]
Thus, when $a>b$, ranking $s$ above $s'$ is more likely to put
a correct response ahead of an incorrect one. Ranking all
prefixes by $q$ makes this choice for every pair and therefore
maximizes AUROC~\citep{clemencon2008ranking}. A strictly
increasing transformation leaves these rankings unchanged.
\end{proof}

\subsection{Assumptions and Extensions}
\label{app:scope}

\begin{itemize}
  % \item \textbf{Terminal state.} With a deterministic grader, $q(s_T) = c \in \{0, 1\}$. The problem is still nontrivial: estimators have finite capacity and must generalize to unseen $(x, y)$, and Theorem~\ref{thm:formal}(ii) says the calibrator and the critic face this finite-sample problem identically. For $t < T$, $q(s_t) \in [0, 1]$ is a genuine probability over the policy's own continuations, which is what the intermediate-step analysis in Section~\ref{sec:exp-qual} uses.
  % \item \textbf{Realizability.} Part (ii) needs no assumption. Part (iii) needs $q \in \mathcal{G}_t$; otherwise the log-loss and squared-loss solutions are different projections of $q$, and under function approximation the iteration of Proposition~\ref{prop:td} converges, when it converges, to a projected fixed point rather than to $V^\pi$ \citep{sutton2018reinforcement}. Both gaps vanish for critics trained with a classification loss \citep{farebrother2024stop}.
  \item \textbf{Which policy.} $V^\pi$ is the value of a fixed policy, including its decoding settings. The RL critic tracks the evolving policy, so it approximates $V^{\pi_K}$ for the final policy $\pi_K$ only up to the lag of its last updates. \method pairs the final critic with $\pi_K$ under training-time decoding and measures the residual gap empirically.
  \item \textbf{Reward.} The same grader must define the RL reward and the correctness label being estimated: the critic estimates $\mathbb{P}_\pi(c = 1 \mid s_t)$ for the $c$ it was trained on. If $R = c + h$ (a KL penalty folded into the reward, or format and length shaping), then $V^\pi(s_t) = q(s_t) + \mathbb{E}_\pi[h \mid s_t]$; the identity is exact only when the critic's reward is correctness, as in verifiable-reward RL where KL enters as a policy loss. For a graded reward $R \in [0, 1]$, all results hold with $q$ replaced by $\mathbb{E}_\pi[R \mid s_t]$.
  \item \textbf{Discounting and optimization.} The identity uses $\gamma = 1$, the default in LLMs RL. With $\gamma < 1$ the critic targets $\mathbb{E}_\pi[\gamma^{T-t} c \mid s_t]$, which discounts long continuations and is no longer a probability of correctness. 
  % The results characterize the population target of the critic loss, not the outcome of optimizing a neural critic with minibatch reuse and value clipping (the latter is inactive at a fixed point, so $V^\pi$ remains a fixed point of PPO's clipped value update); that gap is the one faced by any trained calibrator, and it is what the experiments measure.
  \item \textbf{Distribution.} All statements are under the RL prompt distribution $\mathcal{D}_X$; transfer to other tasks is empirical (Section~\ref{sec:experiments}).
\end{itemize}

\section{Related Work on Training Value Functions for LLMs RL}
\label{app:related}
Value functions are central to RL, from GAE \citep{schulman2016gae} in PPO
\citep{schulman2017ppo} to AlphaGo, whose value network predicts the probability of
winning \citep{silver2016alphago}; classification-trained value functions further
link value learning to the calibration objective \citep{farebrother2024stop}. The
PPO critic is standard in RLHF \citep{ouyang2022training,ziegler2019finetuning}.
Critic-free methods drop the value network for Monte Carlo baselines to save memory
and avoid difficult critic training
\citep{shao2024deepseekmath,deepseekai2025r1,ahmadian2024basics,kazemnejad2024vineppo, wu2026teaching},
but recent work extensively shows well-trained critics give finer credit assignment, lower
variance, and higher sample efficiency for long-CoT reasoning and agentic settings
\citep{yuan2025vcppo,yue2025vapo,hu2025openreasoner,zhou2024archer}. Value models are
also reused for search \citep{liu2023dont} and used for obtaining correctness probabilities in
process reward models \citep{yu2024ovm,wang2024mathshepherd}. 
In this work, we show the critic is
additionally a calibrated, general-purpose confidence estimator for LLMs.
% ============================================================
% Appendices B (experimental details) and C (full results).
% Place after the theory appendix (Appendix A). Needs \usepackage{booktabs,multirow}.
% Numbers are generated from the results sheet; "--" = not run.
% [TODO: ...] marks facts to confirm before submission.
% ============================================================

\section{Experimental Details}
\label{app:exp-details}

\paragraph{Model.}
We initialize the actor and critic from Qwen3-8B or Nemotron-Cascade-8B, with a randomly initialized scalar value head for the critic.
Thinking mode remains enabled under the default chat template.
All prompts request a final answer in $\backslash\mathrm{boxed}\{\}$.

\paragraph{PPO configuration.}
We train with PPO and GAE using a binary correctness reward applied only at the final response token.
With $\gamma=\lambda=1$, the critic's return target equals the observed correctness label at every response position.
KL regularization is applied only to the actor loss, leaving the reward unchanged, and the critic uses unclipped MSE.
Training runs for 300 steps without repeating prompts.
Table~\ref{tab:hparams} summarizes the key settings.

\begin{table}[t]
\centering
\footnotesize
\renewcommand{\arraystretch}{1.08}
\setlength{\tabcolsep}{3pt}
\begin{tabular*}{\columnwidth}{
  @{\extracolsep{\fill}}ll@{}
}
\toprule
\textbf{Setting} & \textbf{Value} \\
\midrule
Algorithm & PPO with GAE \\
Discount $\gamma$ / GAE $\lambda$ & $1.0$ / $1.0$ \\
Reward & Terminal correctness, $\{0,1\}$ \\
Actor / critic learning rate & $10^{-6}$ / $2\times10^{-5}$ \\
Learning-rate schedule & Constant \\
Actor clip ratio & $0.2$ \\
Actor KL loss coefficient & $10^{-3}$ \\
Reward KL coefficient & $0$ \\
Critic loss & Unclipped MSE \\
Entropy coefficient & $0$ \\
PPO epochs per batch & $1$ (actor and critic) \\
Gradient clipping & $1.0$ \\
\midrule
Prompts per step & $16$ \\
Rollouts per prompt & $64$ \\
PPO mini-batch size & $16$ \\
Max prompt / response tokens & $2{,}048$ / $16{,}384$ \\
Sampling temperature / top-$p$ & $1.0$ / $1.0$ \\
\bottomrule
\end{tabular*}
\caption{Key PPO hyperparameters. KL regularization affects only the actor objective.}
\label{tab:hparams}
\end{table}

\paragraph{Evaluation.}
We use the same sampling temperature, top-$p$, and response-length limit as in training, drawing four responses per prompt.
Correctness is determined by rule-based grading, with truncated responses counted as incorrect.
Confidence is read from the critic's scalar output at the last non-padding token.
We compute AUROC from raw outputs and compute ECE (15 equal-width bins) and Brier score after clipping outputs to $[0,1]$.

\section{Full Results}
\label{app:full-results}

Tables~\ref{tab:full-math}--\ref{tab:full-tool} report AUROC, ECE and Brier score for every math, code and agentic tool calling benchmark and both LLMs.

\begin{table*}[t]
\centering
\footnotesize
\renewcommand{\arraystretch}{1.1}
\setlength{\tabcolsep}{2.5pt}
\resizebox{\textwidth}{!}{%
\begin{tabular}{c c l *{10}{c}}
\toprule
& & & \multicolumn{6}{c}{In distribution} & \multicolumn{2}{c}{Out of distribution} & \multicolumn{2}{c}{Average} \\
\cmidrule(lr){4-9} \cmidrule(lr){10-11} \cmidrule(lr){12-13}
& & Method & synth.\ math & olympiads & synth.\ amc & cn\_k12 & aops & amc\_aime & AIME 24 & AIME 25 & ID & OOD \\
\midrule
& & Verbalized & 0.70 & 0.63 & 0.69 & 0.62 & 0.63 & 0.67 & 0.75 & 0.70 & 0.66 & 0.73 \\
& & Self-consistency & 0.67 & 0.73 & 0.64 & 0.68 & 0.77 & 0.72 & 0.91 & \textbf{0.96} & 0.70 & \textbf{0.94} \\
& & R-Tuning & 0.86 & 0.89 & 0.93 & 0.87 & 0.87 & \textbf{0.91} & 0.91 & 0.85 & 0.89 & 0.88 \\
& & Probe & 0.86 & 0.90 & 0.92 & 0.87 & 0.87 & 0.89 & 0.89 & 0.89 & 0.88 & 0.89 \\
\rowcolor{lightpink} \cellcolor{white} & \cellcolor{white}\multirow{-5}{*}{AUC$\uparrow$} & \method{} (ours) & \textbf{0.89} & \textbf{0.93} & \textbf{0.95} & \textbf{0.88} & \textbf{0.89} & 0.89 & \textbf{0.92} & 0.95 & \textbf{0.91} & \textbf{0.94} \\
\cmidrule(l){3-13}
& & Verbalized & 0.53 & 0.56 & 0.70 & 0.46 & 0.58 & 0.37 & 0.17 & 0.21 & 0.56 & 0.19 \\
& & Self-consistency & 0.45 & 0.40 & 0.63 & 0.39 & 0.35 & 0.27 & \textbf{0.06} & 0.25 & 0.44 & 0.16 \\
& & R-Tuning & 0.09 & 0.03 & 0.05 & 0.03 & \textbf{0.03} & 0.08 & 0.25 & 0.22 & 0.05 & 0.24 \\
& & Probe & 0.06 & 0.04 & 0.06 & 0.03 & 0.05 & 0.09 & 0.25 & 0.24 & 0.05 & 0.24 \\
\rowcolor{lightpink} \cellcolor{white} & \cellcolor{white}\multirow{-5}{*}{ECE$\downarrow$} & \method{} (ours) & \textbf{0.04} & \textbf{0.02} & \textbf{0.04} & \textbf{0.02} & \textbf{0.03} & \textbf{0.07} & \textbf{0.06} & \textbf{0.20} & \textbf{0.03} & \textbf{0.13} \\
\cmidrule(l){3-13}
& & Verbalized & 0.50 & 0.54 & 0.65 & 0.45 & 0.55 & 0.35 & 0.19 & 0.23 & 0.53 & 0.21 \\
& & Self-consistency & 0.44 & 0.37 & 0.60 & 0.38 & 0.34 & 0.27 & \textbf{0.05} & \textbf{0.06} & 0.42 & \textbf{0.06} \\
& & R-Tuning & \textbf{0.16} & 0.13 & \textbf{0.10} & 0.14 & \textbf{0.14} & 0.12 & 0.15 & 0.20 & \textbf{0.13} & 0.18 \\
& & Probe & 0.18 & 0.15 & 0.11 & 0.15 & 0.15 & 0.13 & 0.16 & 0.18 & 0.15 & 0.17 \\
\rowcolor{lightpink} \cellcolor{white}\multirow{-15}{*}{\rotatebox[origin=c]{90}{Qwen3-8B}} & \cellcolor{white}\multirow{-5}{*}{Brier$\downarrow$} & \method{} (ours) & 0.19 & \textbf{0.12} & \textbf{0.10} & \textbf{0.13} & \textbf{0.14} & \textbf{0.11} & 0.15 & 0.14 & 0.14 & 0.15 \\
\midrule
& & Verbalized & 0.67 & 0.63 & 0.61 & 0.60 & 0.60 & 0.64 & 0.68 & 0.72 & 0.62 & 0.70 \\
& & Self-consistency & 0.81 & 0.85 & 0.78 & 0.84 & 0.86 & 0.86 & 0.95 & 0.98 & 0.83 & 0.97 \\
& & R-Tuning & 0.84 & 0.87 & 0.86 & 0.83 & 0.85 & 0.88 & 0.88 & 0.87 & 0.85 & 0.88 \\
& & Probe & 0.83 & 0.86 & 0.85 & 0.82 & 0.83 & 0.85 & 0.89 & 0.84 & 0.84 & 0.87 \\
\rowcolor{lightpink} \cellcolor{white} & \cellcolor{white}\multirow{-5}{*}{AUC$\uparrow$} & \method{} (ours) & \textbf{0.86} & \textbf{0.88} & \textbf{0.89} & \textbf{0.86} & \textbf{0.87} & \textbf{0.89} & \textbf{0.99} & \textbf{0.99} & \textbf{0.87} & \textbf{0.99} \\
\cmidrule(l){3-13}
& & Verbalized & 0.42 & 0.48 & 0.58 & 0.39 & 0.54 & 0.27 & \textbf{0.19} & 0.33 & 0.47 & 0.26 \\
& & Self-consistency & 0.25 & 0.20 & 0.35 & 0.17 & 0.18 & 0.14 & 0.20 & 0.16 & 0.23 & 0.18 \\
& & R-Tuning & 0.04 & 0.02 & 0.05 & 0.05 & \textbf{0.01} & \textbf{0.10} & 0.21 & 0.18 & 0.04 & 0.20 \\
& & Probe & 0.03 & 0.04 & \textbf{0.04} & 0.06 & 0.03 & 0.11 & 0.23 & 0.14 & 0.04 & 0.19 \\
\rowcolor{lightpink} \cellcolor{white} & \cellcolor{white}\multirow{-5}{*}{ECE$\downarrow$} & \method{} (ours) & \textbf{0.02} & \textbf{0.01} & \textbf{0.04} & \textbf{0.04} & 0.03 & \textbf{0.10} & 0.22 & \textbf{0.12} & \textbf{0.03} & \textbf{0.17} \\
\cmidrule(l){3-13}
& & Verbalized & 0.41 & 0.45 & 0.53 & 0.40 & 0.50 & 0.30 & 0.25 & 0.32 & 0.45 & 0.28 \\
& & Self-consistency & 0.25 & 0.19 & 0.30 & 0.19 & 0.17 & 0.17 & 0.18 & 0.19 & 0.22 & 0.19 \\
& & R-Tuning & \textbf{0.16} & 0.14 & 0.13 & 0.17 & 0.14 & 0.15 & 0.17 & 0.17 & 0.15 & 0.17 \\
& & Probe & 0.17 & 0.14 & 0.13 & 0.18 & 0.15 & 0.17 & 0.18 & 0.18 & 0.15 & 0.18 \\
\rowcolor{lightpink} \cellcolor{white}\multirow{-15}{*}{\rotatebox[origin=c]{90}{\parbox{2.2cm}{\centering Nemotron-\\Cascade-8B}}} & \cellcolor{white}\multirow{-5}{*}{Brier$\downarrow$} & \method{} (ours) & \textbf{0.16} & \textbf{0.13} & \textbf{0.12} & \textbf{0.16} & \textbf{0.13} & \textbf{0.14} & \textbf{0.15} & \textbf{0.13} & \textbf{0.14} & \textbf{0.14} \\
\bottomrule
\end{tabular}%
}
\caption{Math, per benchmark. ID columns are the six NuminaMath sources used for RL. OOD is AIME 2024 and 2025. Best per column within each model and metric in bold.}
\label{tab:full-math}
\end{table*}

\begin{table*}[t]
\centering
\footnotesize
\renewcommand{\arraystretch}{1.1}
\setlength{\tabcolsep}{3pt}
\newcolumntype{E}{>{\centering\arraybackslash}p{1.2cm}}
\begin{tabular}{c c l *{8}{E}}
\toprule
& & & \multicolumn{3}{c}{In distribution} & \multicolumn{3}{c}{Out of distribution} & \multicolumn{2}{c}{Average} \\
\cmidrule(lr){4-6} \cmidrule(lr){7-9} \cmidrule(lr){10-11}
& & Method & TACO & \makecell{Code-\\Contests} & \makecell{Code-\\forces} & \makecell{LCB\\Easy} & \makecell{LCB\\Medium} & \makecell{LCB\\Hard} & ID & OOD \\
\midrule
& & Verbalized & 0.53 & 0.69 & 0.65 & 0.52 & 0.65 & 0.65 & 0.60 & 0.61 \\
& & Self-consistency & 0.66 & 0.90 & 0.90 & 0.83 & 0.92 & 0.92 & 0.78 & 0.89 \\
& & R-Tuning & 0.87 & 0.86 & 0.84 & \textbf{0.84} & 0.89 & 0.84 & 0.86 & 0.85 \\
& & Probe & 0.90 & 0.87 & 0.85 & 0.81 & 0.89 & 0.85 & 0.88 & 0.84 \\
\rowcolor{lightpink} \cellcolor{white} & \cellcolor{white}\multirow{-5}{*}{AUC$\uparrow$} & \method{} (ours) & \textbf{0.93} & \textbf{0.93} & \textbf{0.91} & \textbf{0.84} & \textbf{0.94} & \textbf{0.95} & \textbf{0.93} & \textbf{0.91} \\
\cmidrule(l){3-11}
& & Verbalized & 0.37 & 0.36 & 0.36 & 0.04 & 0.15 & 0.48 & 0.37 & 0.28 \\
& & Self-consistency & 0.30 & 0.14 & 0.14 & \textbf{0.01} & \textbf{0.09} & 0.16 & 0.22 & 0.10 \\
& & R-Tuning & 0.05 & 0.04 & 0.06 & 0.07 & 0.14 & 0.07 & 0.05 & 0.08 \\
& & Probe & 0.07 & 0.05 & 0.05 & 0.11 & 0.15 & 0.07 & 0.06 & 0.10 \\
\rowcolor{lightpink} \cellcolor{white} & \cellcolor{white}\multirow{-5}{*}{ECE$\downarrow$} & \method{} (ours) & \textbf{0.04} & \textbf{0.03} & \textbf{0.04} & 0.03 & \textbf{0.09} & \textbf{0.06} & \textbf{0.04} & \textbf{0.06} \\
\cmidrule(l){3-11}
& & Verbalized & 0.40 & 0.36 & 0.37 & 0.04 & 0.21 & 0.42 & 0.38 & 0.26 \\
& & Self-consistency & 0.31 & 0.14 & 0.15 & \textbf{0.02} & \textbf{0.08} & 0.16 & 0.23 & 0.10 \\
& & R-Tuning & 0.14 & 0.16 & 0.16 & 0.04 & 0.13 & 0.14 & 0.15 & 0.11 \\
& & Probe & 0.13 & 0.15 & 0.16 & 0.08 & 0.13 & 0.14 & 0.14 & 0.12 \\
\rowcolor{lightpink} \cellcolor{white}\multirow{-15}{*}{\rotatebox[origin=c]{90}{Qwen3-8B}} & \cellcolor{white}\multirow{-5}{*}{Brier$\downarrow$} & \method{} (ours) & \textbf{0.11} & \textbf{0.11} & \textbf{0.12} & 0.05 & 0.12 & \textbf{0.09} & \textbf{0.11} & \textbf{0.08} \\
\midrule
& & Verbalized & 0.60 & 0.57 & 0.55 & 0.59 & 0.61 & 0.61 & 0.58 & 0.60 \\
& & Self-consistency & 0.66 & 0.87 & 0.88 & \textbf{0.83} & 0.92 & 0.79 & 0.76 & \textbf{0.88} \\
& & R-Tuning & 0.91 & 0.88 & 0.87 & 0.59 & 0.92 & 0.88 & 0.89 & 0.79 \\
& & Probe & 0.91 & 0.88 & 0.87 & 0.62 & 0.90 & 0.88 & 0.89 & 0.80 \\
\rowcolor{lightpink} \cellcolor{white} & \cellcolor{white}\multirow{-5}{*}{AUC$\uparrow$} & \method{} (ours) & \textbf{0.93} & \textbf{0.91} & \textbf{0.90} & 0.65 & \textbf{0.99} & \textbf{0.91} & \textbf{0.91} & 0.84 \\
\cmidrule(l){3-11}
& & Verbalized & 0.33 & 0.35 & 0.33 & 0.09 & 0.10 & 0.55 & 0.33 & 0.32 \\
& & Self-consistency & 0.31 & 0.18 & 0.14 & \textbf{0.01} & \textbf{0.09} & 0.27 & 0.24 & 0.20 \\
& & R-Tuning & \textbf{0.04} & \textbf{0.02} & 0.04 & 0.12 & 0.16 & 0.04 & \textbf{0.03} & 0.09 \\
& & Probe & \textbf{0.04} & \textbf{0.02} & \textbf{0.02} & 0.12 & 0.15 & 0.04 & \textbf{0.03} & 0.09 \\
\rowcolor{lightpink} \cellcolor{white} & \cellcolor{white}\multirow{-5}{*}{ECE$\downarrow$} & \method{} (ours) & \textbf{0.04} & \textbf{0.02} & \textbf{0.02} & 0.13 & 0.14 & \textbf{0.03} & \textbf{0.03} & \textbf{0.08} \\
\cmidrule(l){3-11}
& & Verbalized & 0.34 & 0.35 & 0.35 & 0.02 & 0.16 & 0.49 & 0.35 & 0.28 \\
& & Self-consistency & 0.32 & 0.18 & 0.16 & \textbf{0.01} & \textbf{0.08} & 0.25 & 0.25 & 0.16 \\
& & R-Tuning & 0.12 & 0.14 & 0.15 & 0.03 & 0.11 & \textbf{0.13} & 0.13 & \textbf{0.09} \\
& & Probe & 0.12 & 0.14 & 0.15 & 0.03 & 0.12 & \textbf{0.13} & 0.13 & \textbf{0.09} \\
\rowcolor{lightpink} \cellcolor{white}\multirow{-15}{*}{\rotatebox[origin=c]{90}{\parbox{2.2cm}{\centering Nemotron-\\Cascade-8B}}} & \cellcolor{white}\multirow{-5}{*}{Brier$\downarrow$} & \method{} (ours) & \textbf{0.11} & \textbf{0.12} & \textbf{0.13} & 0.03 & 0.12 & \textbf{0.13} & \textbf{0.12} & \textbf{0.09} \\
\bottomrule
\end{tabular}
\caption{Code, per benchmark. ID columns are the three sources used for RL. OOD is LiveCodeBench (LCB) by difficulty. Best per column within each model and metric in bold.}
\label{tab:full-code}
\end{table*}

\begin{table*}[t]
\centering
\footnotesize
\renewcommand{\arraystretch}{1.1}
\setlength{\tabcolsep}{5pt}
\begin{tabular}{c c l *{6}{c}}
\toprule
& & & \multicolumn{1}{c}{In distribution} & \multicolumn{3}{c}{Out of distribution} & \multicolumn{2}{c}{Average} \\
\cmidrule(lr){4-4} \cmidrule(lr){5-7} \cmidrule(lr){8-9}
& & Method & ToolRL & API-Bank L1 & API-Bank L2 & API-Bank L3 & ID & OOD \\
\midrule
& & Verbalized & 0.47 & 0.52 & 0.65 & 0.68 & 0.47 & 0.57 \\
& & Self-consistency & 0.71 & 0.92 & 0.92 & 0.85 & 0.71 & 0.91 \\
& & R-Tuning & 0.95 & 0.91 & 0.92 & 0.86 & 0.95 & 0.90 \\
& & Probe & 0.93 & 0.84 & 0.92 & 0.85 & 0.93 & 0.85 \\
\rowcolor{lightpink} \cellcolor{white} & \cellcolor{white}\multirow{-5}{*}{AUC$\uparrow$} & \method{} (ours) & \textbf{0.97} & \textbf{0.93} & \textbf{0.95} & \textbf{0.89} & \textbf{0.97} & \textbf{0.92} \\
\cmidrule(l){3-9}
& & Verbalized & 0.37 & 0.21 & 0.41 & 0.51 & 0.37 & 0.30 \\
& & Self-consistency & 0.38 & 0.23 & 0.33 & 0.49 & 0.38 & 0.30 \\
& & R-Tuning & 0.07 & 0.20 & 0.17 & 0.16 & 0.07 & 0.19 \\
& & Probe & 0.12 & 0.12 & 0.18 & 0.21 & 0.12 & 0.15 \\
\rowcolor{lightpink} \cellcolor{white} & \cellcolor{white}\multirow{-5}{*}{ECE$\downarrow$} & \method{} (ours) & \textbf{0.06} & \textbf{0.09} & \textbf{0.12} & \textbf{0.15} & \textbf{0.06} & \textbf{0.11} \\
\cmidrule(l){3-9}
& & Verbalized & 0.39 & 0.26 & 0.42 & 0.50 & 0.39 & 0.33 \\
& & Self-consistency & 0.36 & 0.23 & 0.32 & 0.48 & 0.36 & 0.30 \\
& & R-Tuning & 0.08 & \textbf{0.17} & 0.14 & 0.14 & 0.08 & 0.16 \\
& & Probe & 0.10 & 0.18 & 0.14 & 0.18 & 0.10 & 0.18 \\
\rowcolor{lightpink} \cellcolor{white}\multirow{-15}{*}{\rotatebox[origin=c]{90}{Qwen3-8B}} & \cellcolor{white}\multirow{-5}{*}{Brier$\downarrow$} & \method{} (ours) & \textbf{0.05} & \textbf{0.17} & \textbf{0.13} & \textbf{0.11} & \textbf{0.05} & \textbf{0.15} \\
\midrule
& & Verbalized & 0.56 & 0.49 & 0.48 & 0.34 & 0.56 & 0.46 \\
& & Self-consistency & 0.81 & \textbf{0.80} & 0.75 & 0.68 & 0.81 & \textbf{0.77} \\
& & R-Tuning & 0.86 & 0.78 & 0.71 & 0.54 & 0.86 & 0.72 \\
& & Probe & 0.86 & 0.77 & 0.69 & 0.66 & 0.86 & 0.74 \\
\rowcolor{lightpink} \cellcolor{white} & \cellcolor{white}\multirow{-5}{*}{AUC$\uparrow$} & \method{} (ours) & \textbf{0.88} & 0.79 & \textbf{0.76} & \textbf{0.74} & \textbf{0.88} & \textbf{0.77} \\
\cmidrule(l){3-9}
& & Verbalized & 0.32 & 0.26 & 0.36 & 0.31 & 0.32 & 0.28 \\
& & Self-consistency & 0.22 & 0.22 & 0.22 & 0.32 & 0.22 & 0.24 \\
& & R-Tuning & 0.04 & 0.23 & \textbf{0.19} & 0.24 & 0.04 & 0.23 \\
& & Probe & 0.04 & 0.33 & 0.28 & 0.23 & 0.04 & 0.30 \\
\rowcolor{lightpink} \cellcolor{white} & \cellcolor{white}\multirow{-5}{*}{ECE$\downarrow$} & \method{} (ours) & \textbf{0.03} & \textbf{0.21} & 0.20 & \textbf{0.21} & \textbf{0.03} & \textbf{0.21} \\
\cmidrule(l){3-9}
& & Verbalized & 0.36 & 0.33 & 0.40 & 0.40 & 0.36 & 0.35 \\
& & Self-consistency & 0.22 & 0.28 & \textbf{0.26} & 0.30 & 0.22 & 0.28 \\
& & R-Tuning & 0.15 & 0.24 & \textbf{0.26} & 0.31 & 0.15 & 0.26 \\
& & Probe & 0.14 & 0.31 & 0.31 & 0.29 & 0.14 & 0.30 \\
\rowcolor{lightpink} \cellcolor{white}\multirow{-15}{*}{\rotatebox[origin=c]{90}{\parbox{2.2cm}{\centering Nemotron-\\Cascade-8B}}} & \cellcolor{white}\multirow{-5}{*}{Brier$\downarrow$} & \method{} (ours) & \textbf{0.13} & \textbf{0.23} & \textbf{0.26} & \textbf{0.28} & \textbf{0.13} & \textbf{0.24} \\
\bottomrule
\end{tabular}
\caption{Agentic tool calling, per benchmark. ID is held-out ToolRL. OOD is API-Bank levels 1 to 3. Best per column within each model and metric in bold.}
\label{tab:full-tool}
\end{table*}

\end{document}